\documentclass[lettersize,journal]{IEEEtran}
\usepackage{amsmath,amsfonts}
\usepackage{algorithmic}
\usepackage{algorithm}
\usepackage{array}
\usepackage{textcomp}
\usepackage{url}            %
\usepackage{booktabs}       %
\usepackage{amsfonts}       %
\usepackage{nicefrac}       %
\usepackage{microtype}      %
\usepackage{graphicx}
\usepackage{color}
\usepackage{caption}
\usepackage{colortbl}
\usepackage{dsfont}
\usepackage{subfigure}
\usepackage{algorithm}
\usepackage{graphicx}
\usepackage{algorithmic}
\usepackage{amsmath,amssymb}
\usepackage{multirow}
\usepackage{mathtools}
\usepackage{amsthm}
\newtheorem{theorem}{Theorem}
\newtheorem{definition}{Definition}
\newtheorem{proposition}{Proposition}
\newtheorem{remark}{Remark}

\newtheorem{lemma}{Lemma}
\usepackage{stfloats}
\usepackage{url}
\usepackage{verbatim}
\usepackage{graphicx}
\usepackage{cite}
\begin{document}

\title{Generalist++: A Meta-learning Framework for Mitigating Trade-off in Adversarial Training}

\author{Yisen Wang,~\IEEEmembership{Member,~IEEE,}
        Yichuan Mo, Hongjun Wang, Junyi Li,
        Zhouchen Lin,~\IEEEmembership{Fellow,~IEEE}%
\IEEEcompsocitemizethanks{\IEEEcompsocthanksitem Yisen Wang, Yichuan Mo, and Zhouchen Lin are with State Key Lab of General Artificial Intelligence, School of Intelligence Science and Technology, Peking University. Emails: yisen.wang@pku.edu.cn, mo666666@stu.pku.edu.cn, zlin@pku.edu.cn. \protect\\

\IEEEcompsocthanksitem Junyi Li is with School of Mathematical Sciences, Peking University. Email: 142857pony@stu.pku.edu.cn. \protect\\

\IEEEcompsocthanksitem Hongjun Wang is with School of Computing and Data Science, The University of Hong Kong. Email: hjwang@connect.hku.hk. His main contribution
to this work was done during his intern at Peking University. \protect\\}%
\thanks{}}

\markboth{Journal of \LaTeX\ Class Files,~Vol.~14, No.~8, August~2021}%
{Shell \MakeLowercase{\textit{et al.}}: A Sample Article Using IEEEtran.cls for IEEE Journals}

\IEEEpubid{0000--0000/00\$00.00~\copyright~2021 IEEE}

\maketitle

\begin{abstract}
Despite the rapid progress of neural networks, they remain highly vulnerable to adversarial examples, for which adversarial training (AT) is currently the most effective defense. While AT has been extensively studied, its practical applications expose two major limitations: natural accuracy tends to degrade significantly compared with standard training, and robustness does not transfer well across attacks crafted under different norm constraints. Unlike prior works that attempt to address only one issue within a single network, we propose to partition the overall generalization goal into multiple sub-tasks, each assigned to a dedicated base learner. By specializing in its designated objective, each base learner quickly becomes an expert in its field. In the later stages of training, we interpolate their parameters to form a knowledgeable global learner, while periodically redistributing the global parameters back to the base learners to prevent their optimization trajectories from drifting too far from the shared target. We term this framework Generalist and introduce three variants tailored to different application scenarios. Both theoretical analysis and extensive experiments demonstrate that Generalist achieves lower generalization error and significantly alleviates the trade-off problems compared with baseline methods. Our results suggest that Generalist provides a promising step toward developing fully robust classifiers in the future.

\end{abstract}


\section{Introduction}
\label{sec:intro}

\IEEEPARstart{I}{n} recent years, deep learning has achieved remarkable progress across a wide range of domains, including image classification~\cite{DBLP:conf/cvpr/HeZRS16,huang2017densely,DBLP:journals/corr/ZagoruykoK16}, machine translation~\cite{devlin2019bert,achiam2023gpt}, and speech synthesis~\cite{guo2025evolution,krug2025precisely,lian2025cauchy}. Despite these advances, deep models remain highly vulnerable to adversarial attacks~\cite{bai2025rat,wang2024adversarial,DBLP:conf/cvpr/DongLPS0HL18}, where imperceptible perturbations deliberately added to inputs can drastically degrade performance. Such attacks not only undermine the utility of these systems but may also cause severe consequences in safety-critical applications, such as medical misdiagnosis~\cite{kanca2025enhancing} or traffic accidents~\cite{lu2024adversarial}. To counter these risks, a variety of defense strategies have been proposed, among which adversarial training (AT)\cite{PGD,wang2019dynamic,wang2020improving,mo2022adversarial,sui2025isdat,zhao2025adversarial,liu2025parameter} has emerged as the most effective. AT dynamically generates adversarial examples during training and incorporates them into the optimization process. Despite its effectiveness, AT still suffers from severe trade-off problems that hinder its broader deployment. On the one hand, there exists an \textbf{outer trade-off} between natural and robust accuracy: improving robustness against adversarial perturbations usually comes at the cost of reduced performance on clean samples, as illustrated in Figure \ref{fig:first}(a). On the other hand, an \textbf{inner trade-off} arises across different norm constraints, where enhancing robustness against $\ell_\infty$-bounded attacks typically compromises robustness against $\ell_2$-bounded ones, as shown in Figure \ref{fig:first}(b). These dilemmas have significantly limited the practical applicability of AT in real-world scenarios.

\begin{figure}[!t]
\centering
\subfigure[Trade-off between robustness and accuracy]
{\includegraphics[width=0.48\linewidth]{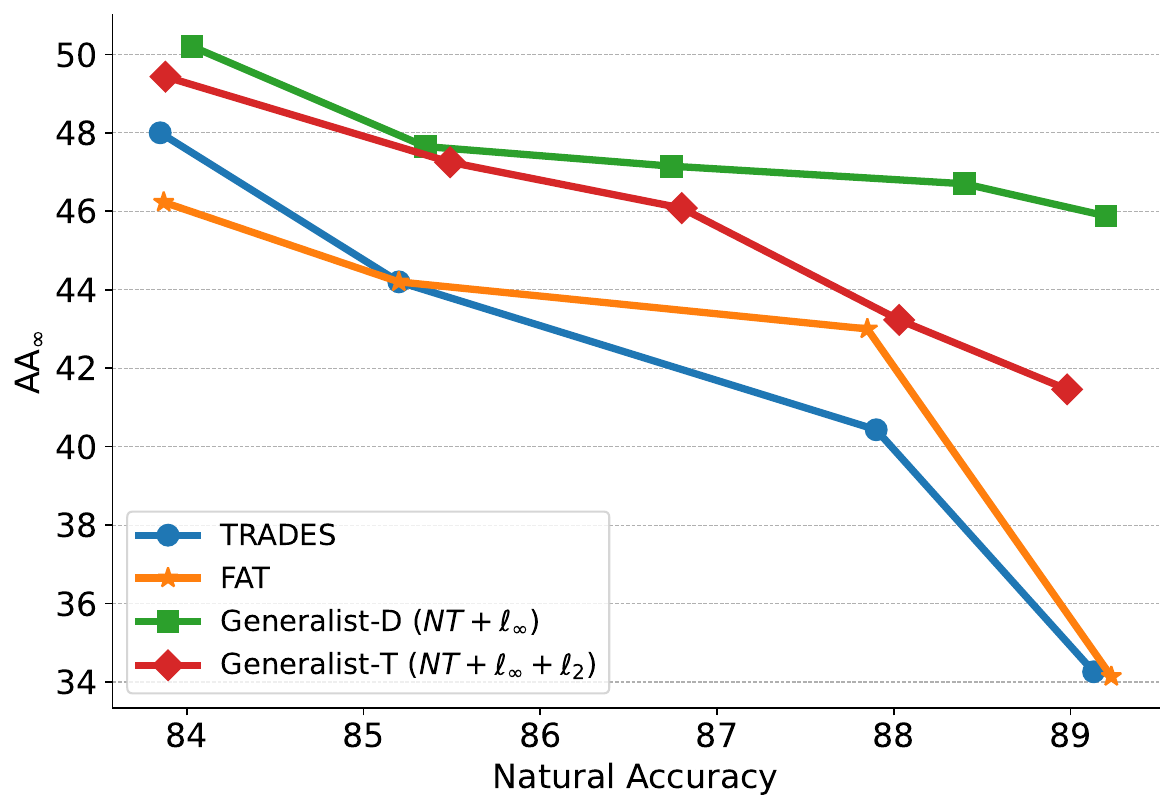}}
 \subfigure[Trade-off between $\ell_\infty$ and $\ell_2$ robustness ]	{\includegraphics[width=0.48\linewidth]{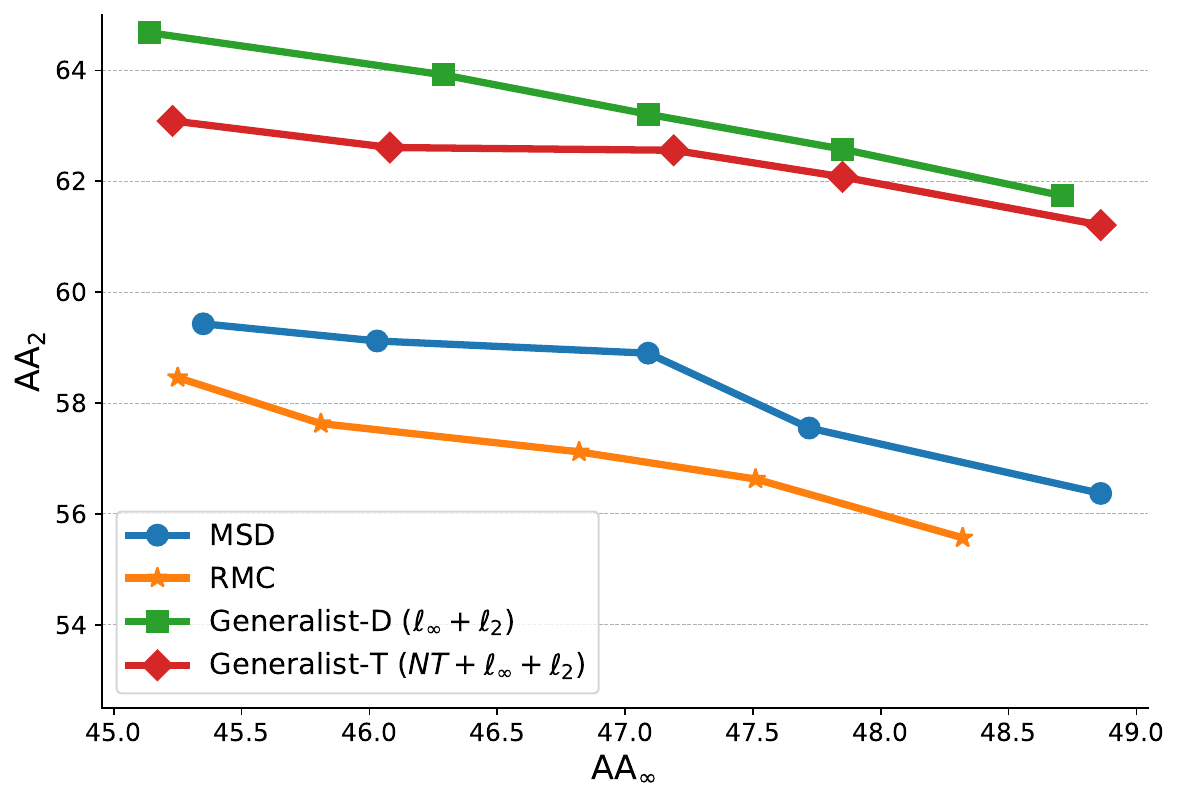}}

\caption{Comparison with current variants of AT that aim at achieving a better trade-off. Note that the baseline for comparison is different in (a) and (b) because existing methods typically address one problem at a time. We compare {Generalist} against their respective areas of expertise. Results show that Generalist achieves strong performance when focusing on a single trade-off issue (see Generalist-D). Moreover, when addressing two issues simultaneously, Generalist outperforms existing baselines in both aspects (see Generalist-T). The improvement is notable since we only use the naive cross-entropy loss without increasing model size.}
\label{fig:first}
\end{figure}

\IEEEpubidadjcol

Although prior works have studied these issues, most frameworks are designed to alleviate only one trade-off at a time. For example, to address the accuracy–robustness trade-off, some works provide theoretical analyses \cite{DBLP:conf/iclr/TsiprasSETM19, trades}, while subsequent approaches attempt indirect solutions such as incorporating additional labeled or unlabeled data~\cite{DBLP:conf/nips/AlayracUHFSK19,DBLP:conf/nips/NajafiMKM19,DBLP:conf/nips/CarmonRSDL19,RST}, adjusting the perturbation bounds~\cite{fat,PART,ge2025rethinking} or selectively optimizing the specific layers~\cite{gowdaconserve}. For the trade-off across norm bounds, remedies include augmenting training inputs with generative models~\cite{MNG} or sampling diverse adversarial examples via advanced strategies~\cite{E_AT,MSD}. However, these methods remain data-centric and problem-specific, without addressing the root cause from the perspective of the training paradigm.

\begin{figure}[!t]
    \centering
\includegraphics[width=0.5\textwidth]{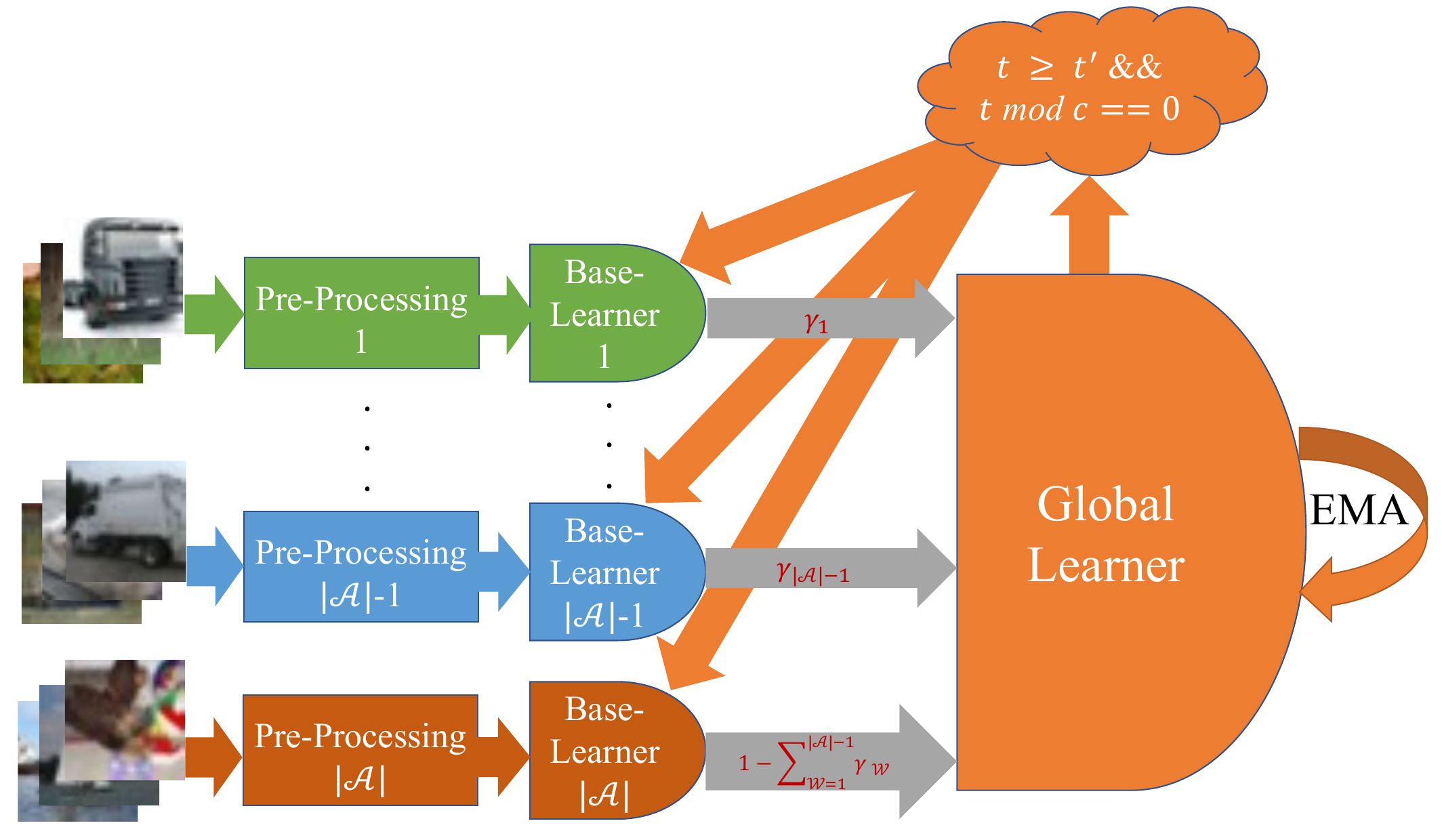}
    \caption{Pipeline of the proposed Generalist. Multiple base learners are trained independently within their respective sub-tasks. A global learner periodically aggregates parameters from the base learners, integrates knowledge, and redistributes the updated parameters back for continued training.} \label{fig:pipeline}
\end{figure}

Inspired by the principle of divide-and-conquer, we propose a novel Generalist paradigm that decouples the objective of adversarial training into multiple sub-tasks. In the case of the natural–robustness trade-off, the subtasks correspond to natural example classification and adversarial example classification, while for the multi-norm robustness trade-off, each subtask corresponds to classification under a single norm constraint. For every subtask, we train a dedicated base learner with task-specific data and configurations while maintaining the same model architecture across all subtasks. The parameters of these base learners are periodically aggregated into a global learner, which then redistributes its knowledge back to the base learners as initialization for continued training. This cyclical process enables the global learner to integrate complementary strengths while allowing each base learner to specialize in its own domain. We term the overall framework Generalist, whose proof-of-concept pipeline is illustrated in Figure~\ref{fig:pipeline}.

Unlike traditional joint training frameworks that attempt to balance multiple objectives simultaneously, Generalist explicitly leverages task-aware specialization. Each base learner can explore the optimal trajectory for its subtask, while the global learner integrates their strengths. Depending on the number of base learners, we instantiate three variants: 1) Generalist-D ($NT+\ell_\infty$): natural + $\ell_\infty$ adversarial training, 2) Generalist-D ($\ell_\infty+\ell_2$): dual-norm adversarial training, and 3) Generalist-T ($NT+\ell_\infty+\ell_2$): triple-task training. We theoretically prove that if the base learners are well trained, the aggregated global learner is guaranteed to achieve lower risk. To our knowledge, Generalist is the first task-aware training paradigm designed to simultaneously alleviate both trade-offs in adversarial training (performance preview in Figure \ref{fig:first}).
The main contributions of this work are summarized as follows:
\begin{itemize}
  \item We introduce a novel Generalist paradigm that addresses both major trade-offs in adversarial training—natural vs. robust accuracy and robustness across different norm bounds—by constructing multiple task-aware base learners rather than relying on joint training.
  \item Our framework allows complete customization of training strategies (\textit{e.g.}, optimization schemes) for each base learner, enabling them to specialize effectively while the global learner integrates their complementary strengths.
  \item We provide extensive experiments on small- and large-scale datasets, demonstrating that Generalist achieves state-of-the-art results in alleviating trade-off problems.
\end{itemize}

The main results of Generalist-D ($NT+\ell_\infty$) were published originally in CVPR 2023 as a highlight paper \cite{wang2023generalist}. 
In this longer article version, we extend it from the following aspects:
\begin{itemize}
   \item We propose two new variants, Generalist-D ($\ell_\infty+\ell_2$) and Generalist-T ($NT+\ell_\infty+\ell_2$) to alleviate the trade-off across multi-norms (Section \ref{sec:method}).
    \item We generalize the theoretical analysis from the two-learner case to an arbitrary number of base learners, showing that parameter aggregation across multiple subtasks leads to a provably lower expected error with tighter generalization guarantees. Moreover, from a stability perspective, we prove that Generalist maintains convergence without amplifying mini-batch randomness, as the global learning dynamics remain governed by per-task stability under mild regularity conditions (Section~\ref{sec:theore}).
    \item We further evaluate Generalist on large-scale datasets and out-of-distribution (OOD) scenarios, demonstrating not only its effectiveness at scale but also its strong transferability to unseen perturbations (Section \ref{sec:largescale} and Section \ref{sec:ood}).
    \item We further include extensive ablation studies and interpretable analyses to investigate the working dynamics of Generalist (Section \ref{sec:mixing}, \ref{sec:customized}, and \ref{sec:visual}).
\end{itemize}

\renewcommand{\arraystretch}{0.85}

\section{Preliminaries and Related Work}
\label{sec:related}
In this section, we provide the necessary background and terminology related to adversarial training and meta-learning.
\label{sec:preliminaries}

\textbf{Notations.} Consider an image classification task with input space $\mathcal{X}$ and output space $\mathcal{Y}$. Let $x \in\mathcal{X} \subseteq \mathbb{R}^{d}$ denote a natural image and $y\in\mathcal{Y}=\{1,2, \ldots, K\}$ denote its corresponding ground-truth label. We denote the natural dataset as $\mathcal{X}\times\mathcal{Y}={(x_i,y_i)}_{i=1}^{n}$, sampled from distribution $\mathcal{D}_1$, and the adversarial dataset as $\mathcal{X}'\times\mathcal{Y}={(x'_i,y_i)}_{i=1}^{n}$, sampled from distribution $\mathcal{D}_2$. A deep neural network (DNN) classifier is represented as $f_{\theta}: \mathcal{X}\rightarrow \mathbb{R}^{K}$, parameterized by $\boldsymbol{\theta}\in \Theta$, which maps any input image to one of the $K$ classes. The objective functions for the natural and adversarial settings are defined as $\ell_{1} \overset{def}{=} \mathcal{D}_{1}\times \Theta \rightarrow [0,\infty)$ and $\ell_{2} \overset{def}{=} \mathcal{D}_{2}\times \Theta \rightarrow [0,\infty)$, respectively. These functions are typically assumed to be positive, bounded, and upper semi-continuous~\cite{DBLP:journals/mor/BlanchetM19,Villani2003TopicsIO,DBLP:conf/colt/BartlettM01}.

\subsection{Adversarial Training and Trade-off Issues}
\textbf{Adversarial Training.} The goal of an adversary is to craft a malicious example $x^{\prime}$ by adding an imperceptible perturbation $\varepsilon \in \mathbb{R}^{d}$ to a natural input $x$. The resulting adversarial example $x^{\prime}$ should remain visually similar to $x$ while inducing misclassification. This perturbation is constrained within a neighborhood of $x$, defined as $\mathbb{B}_{\varepsilon}(x)=\{(x^{\prime},y)\in \mathcal{D}_2 \mid ||x-x^{\prime}||_p \leq \varepsilon\}$, where $p=1,2,\ldots,\infty$ specifies the norm space used for adversarial samples. Adversarial training (AT) defends against such perturbations by generating adversarial examples and optimizing model parameters with respect to them. According to \cite{PGD}, the iterative procedure of AT under an $\ell_p$-norm budget can be summarized as:
\begin{equation}
\small
\label{eqn:AT_basic}
\left\{\begin{array}{l}
{x^{\prime}}^{(t+1)}=\Pi_{\mathbb{B}\left(x, \epsilon\right)}\left({x^{\prime}}^{(t)}+\alpha 
*\frac{\left(\nabla_{x^{\prime}} \ell_2({x^{\prime}}^{(t)}, y; \boldsymbol\theta^t)\right)^{q-1}}{\vert\vert\nabla_{x^{\prime}} \ell_2\left({x^{\prime}}^{(t)}, y; \boldsymbol\theta^t\right)\vert\vert_q^{q-1}}\right) \\
\boldsymbol\theta^{(t+1)}=\boldsymbol\theta^{(t)} - \tau\nabla_{\boldsymbol\theta} \mathbb{E}[\ell_1(x, y; \boldsymbol\theta^t)+\beta\mathcal{R}(x^{\prime}, x, y; \boldsymbol\theta^t)],
\end{array}\right.
\end{equation}
where $\ell_q$ is the dual norm of the  $\ell_p$ norm, $\Pi_{\mathbb{B}\left(x, \epsilon\right)}$ is the projection operator, $\alpha$ is the step size, $\tau$ is the learning rate, and $\mathcal{R}(\cdot)$ is the loss difference of $\ell_2(x^{\prime}, y; \boldsymbol\theta^t) - \ell_1(x, y; \boldsymbol\theta^t)$. The exponent $(\cdot)^{q-1}$ preserves the sign of the gradient, while the trade-off factor $\beta$ balances natural and robust errors. Many AT variants arise from Eq.\ref{eqn:AT_basic}. For instance, $\beta=1$ recovers vanilla PGD training\cite{PGD}, $\beta=1/2$ yields the half-half loss~\cite{DBLP:journals/corr/GoodfellowSS14}, and $\beta=0$ degenerates to standard natural training. Replacing $\mathcal{R}(\cdot)$ with KL divergence or squared error leads to TRADES~\cite{trades} or LSE~\cite{LSE}, respectively.

\textbf{Trade-off Issues with AT.} Although AT is regarded as the most reliable defense~\cite{athalye2018obfuscated}, it faces persistent trade-off challenges. One major problem is the tension between natural and robust accuracy: models trained with AT typically achieve higher robustness at the cost of lower accuracy on clean samples. This phenomenon was first analyzed in \cite{trades,DBLP:conf/iclr/TsiprasSETM19}, with follow-up works attributing it to excessively strong adversarial examples. Methods such as FAT and LSE~\cite{fat,LSE} mitigate this by reducing perturbation strength via fewer iterations or smaller budgets, while others like IAT~\cite{IAT} and AGR~\cite{AGR} normalize AT with natural training loss to stabilize learning. Another challenge is the inconsistency of robustness across norms. Ideally, a robust classifier should withstand attacks under various constraints. However, \cite{tramer2019adversarial} showed that robustness drops sharply when training and evaluation norms differ. Empirical remedies diversify the attack norms during training, leading to techniques such as average-norm operations~\cite{tramer2019adversarial}, steepest ascent updates~\cite{MSD}, random norm selection~\cite{E_AT,MNG}, and logit pairing~\cite{RMC}.

In contrast to these approaches, our proposed framework Generalist addresses both trade-off problems simultaneously within a unified paradigm. Rather than forcing a single model to balance conflicting objectives, we decouple the tasks into separate base learners, each specializing in its own objective, thereby substantially alleviating the inherent trade-offs.

\subsection{Multi-Task Learning and Meta-Learning}

The core idea of multi-task learning (MTL) is to exploit commonalities across tasks by training them jointly, so that shared structures can improve the performance of each individual task~\cite{DBLP:conf/nips/BilenV16,DBLP:conf/cvpr/LuKZCJF17,DBLP:conf/iclr/YangH17,mo2022multi}. Formally, consider a set of assignments $\mathcal{A}=\{\mathcal{D}, \ell\}$ defined by data distributions and loss functions with corresponding models $\left\{\mathcal{M}_{a}\right\}_{a=1}^{n}$ parameterized by $\boldsymbol\theta_{\mathcal{M}_{a}}$. The goal of MTL is to jointly optimize these tasks to obtain task-specific parameters $\boldsymbol\theta_{\mathcal{M}_{a}}^{\star}$:
\begin{equation}
\label{eqn:MTL}
\bigcup_{a=1}^{|\mathcal{A}|} \boldsymbol\theta_{\mathcal{M}_{a}}^{\star}=\underset{\cup_{a=1}^{|\mathcal{A}|} \boldsymbol\theta_{\mathcal{M}_{a}}}{\operatorname{argmin}} \mathbb{E}_{\mathcal{A}}\mathbb{E}_{\mathcal{D}}\ \ell_{a}\left(\mathcal{D}_a; \boldsymbol\theta_{\mathcal{M}_{a}}\right),
\end{equation}
where $\ell_a(\mathcal{D}_a; \boldsymbol\theta_{\mathcal{M}_{a}})$ measures the performance of a model $\boldsymbol\theta_{\mathcal{M}_{a}}$ on dataset $\mathcal{D}_a$. While this joint optimization encourages knowledge sharing, it constrains all tasks to be optimized in a homogeneous fashion.
In contrast, meta-learning emphasizes rapid adaptation, aiming to equip models with the ability to generalize to unseen tasks by leveraging training on related but disjoint sets of tasks~\cite{DBLP:conf/icml/FinnAL17,DBLP:journals/corr/abs-1803-02999}. Suppose the task set $\mathcal{A}$ is split into non-overlapping subsets $\mathcal{V}$ and $\mathcal{W}$. The model is first trained on tasks in $\mathcal{W}$ and then adapted to $\mathcal{V}$, leading to the following formulation:
\begin{equation}
\label{eqn:meta-learning}
\small
\boldsymbol\theta^{\star}=\underset{\boldsymbol\theta}{\operatorname{argmin}} \mathbb{E}_{\mathcal{V}}\mathbb{E}_{\mathcal{D}_{\mathcal{V}}}\ \ell_{\mathcal{V}}\left(\mathcal{D}_{\mathcal{V}}; \underset{\boldsymbol\theta}{\operatorname{argmin}} \mathbb{E}_{\mathcal{W}}\mathbb{E}_{\mathcal{D}_{\mathcal{W}}}\ \ell_{\mathcal{W}}\left(\mathcal{D}_{\mathcal{W}}; \boldsymbol\theta\right)\right).
\end{equation}
Unlike MTL, which optimizes for a set of known tasks, meta-learning is designed to facilitate transfer to previously unseen ones, often through good initialization or update strategies.

Our proposed \emph{Generalist} framework draws inspiration from both paradigms: like MTL, it learns from multiple sources simultaneously, yet unlike MTL, each sub-task can be optimized with heterogeneous strategies; and similar to meta-learning, it leverages shared initialization and periodic aggregation to transfer knowledge across tasks while still allowing base learners to specialize.

\section{The Proposed Framework: Generalist}
\label{sec:method}
Similar to a physical-world generalist who has broad knowledge across many topics and expertise in a few, our proposed Generalist is designed to handle multiple tasks across different domains.

\subsection{Overview}
\label{sec:overview}

Generalist consists of several base learners, each gradually specializing in its own sub-field, while collectively contributing to a global learner that accumulates and redistributes knowledge. The framework operates in two steps: 1) each base learner $\boldsymbol\theta_a$ is optimized on its assigned data distribution $\mathcal{D}_a$, and 2) the parameters of the global learner $\boldsymbol\theta_g$ are periodically aggregated and redistributed to all base learners. Through this continuous interaction, the global learner disseminates accumulated knowledge, while base learners refine their expertise by periodically re-initializing from the global parameters. All base learners and the global learner share the same architecture, i.e., $\mathcal{M}_{1}=\mathcal{M}_{2}=\cdots=\mathcal{M}_{|\mathcal{A}|}$.

Specifically, when $|\mathcal{A}|=2$, we obtain the ``Double" version of Generalist (\textbf{Generalist-D}), aiming to address one single trade-off problem. Similarly, When $|\mathcal{A}|=3$, the ``Triple" version of Generalist (\textbf{Generalist-T}) integrates knowledge from three learners, enabling it not only to balance the trade-off between robustness and natural accuracy but also to achieve strong robustness across different norms. The overall procedures of Generalist-D and Generalist-T are presented in Algorithm~\ref{alg:Generalist-D} and Algorithm~\ref{alg:Generalist-T}, respectively.

\subsection{Task-aware Base Learners}
Given a global data distribution $\mathcal{D}$ for the trade-off problem, as denoted in Section \ref{sec:preliminaries}, $\mathcal{D}_1,\dots, \mathcal{D}_{|\mathcal{A}|}$ are subject to the distribution of training data $\mathcal{D}_{\mathcal{W}}$. The training of base learners corresponds to solving the inner minimization of Eq.~\ref{eqn:meta-learning} over these distributions in a distributed manner:
\begin{equation}
\label{eqn:base learner}
\left\{\boldsymbol\theta_1^{\star},\dots, \boldsymbol\theta_{|\mathcal{A}|}^{\star}\right\}=\underset{\bigcup_{\mathcal{W}=1}^{|\mathcal{A}|}\boldsymbol\theta_{\mathcal{W}}}{\operatorname{argmin}}\mathbb{E}_{\mathcal{D}_{\mathcal{W}}}\ \ell_{\mathcal{W}}\left(\mathcal{D}_{\mathcal{W}}; \boldsymbol\theta_{\mathcal{W}}\right).
\end{equation}
Specifically, during training, each base learner $f_{\boldsymbol\theta_{\mathcal{W}}}$ is assigned a specific subproblem and requires access only to its own data distribution. The base learners operate in a complementary manner: their parameter updates are performed independently, while the global learner periodically aggregates their parameters. The optimization subproblem for each base learner is defined as:
\begin{equation}
\label{eqn:subproblem}
\boldsymbol\theta_{\mathcal{W}}^{\star}=\underset{\boldsymbol\theta}{\operatorname{argmin}}{\mathcal{Z}_{\mathcal{W}}^{T}\left[\mathbb{E}_{\mathcal{W}}(\nabla_{\boldsymbol\theta} \ell_{\mathcal{W}}(\mathcal{D}_{\mathcal{W}}; \boldsymbol\theta_{\mathcal{W}})),\tau_{\mathcal{W}}\right]}.
\end{equation}
where the task-aware optimizer $\mathcal{Z}_{\mathcal{W}}^{T}(\cdot,\cdot)$ searches for the optimal parameters $\boldsymbol\theta_{\mathcal{W}}^{\star}$ for subproblem $\mathcal{W}$ within $T$ rounds. Each base learner can also adopt task-specific loss functions. Although minimizing the 0-1 loss for natural and adversarial errors is theoretically ideal, the problem is NP-hard and computationally intractable. In practice, we employ cross-entropy as a surrogate loss for each $\ell_{\mathcal{W}}$, since it provides a simple yet effective approximation.

\begin{figure*}[!t]
\centering
\begin{minipage}[t]{0.48\linewidth}
\begin{algorithm}[H]
\footnotesize
\caption{\footnotesize{Generalist-D: The double version of Generalist for leveraging learning trajectory with respect to two task-aware base learners to alleviate one trade-off problem.}}
   \label{alg:Generalist-D}
   \algsetup{linenosize=\tiny}
   \begin{algorithmic}
   \STATE {\bfseries Input:} A DNN classifier $f(\cdot)$ with initial learnable parameters $\boldsymbol\theta_g$ for the global learner and parameters $\boldsymbol\theta_1, \boldsymbol\theta_2$ for each base learner with objective functions $\ell_1, \ell_2$, learning rate $\tau_1, \tau_2$, optimizers $\mathcal{Z}_1, \mathcal{Z}_2$; functions for the generation of adversarial samples $G_\infty, G_2$; number of iterations $T$; data distribution $\mathcal{D}$; exponential decay rates for ensembling $\alpha^{\prime}=0.999$; mixing ratio $\gamma_1$; starting point and frequency of communication $t^{\prime}, c$; \texttt{Mode} of performing Generalist-D. 
   \STATE Initialize $\boldsymbol\theta_g, \boldsymbol\theta_1, \boldsymbol\theta_2$ in $\Theta$ space.
   \FOR {t $ \leftarrow 1, 2, \cdots , T$}
    \STATE Sample a minibatch $(x, y)$ from the data distribution $\mathcal{D}$.
    \STATE (Optional) Performing model ensembling, data augmentation or label smoothing, etc.
    \STATE $\boldsymbol\theta_1 \leftarrow \mathcal{Z}_1\left[\mathbb{E}_{(x,y)}(\nabla_{\boldsymbol\theta_1} \ell_1(G_\infty(x),y; \boldsymbol\theta_{1})),\tau_{1}\right]$
    \STATE (Optional) Performing model ensembling, data augmentation or label smoothing, etc.
    \IF{\texttt{Mode} == ``$\ell_\infty+\ell_2$"}
    \STATE $\boldsymbol\theta_2 \leftarrow \mathcal{Z}_2\left[\mathbb{E}_{(x,y)}(\nabla_{\boldsymbol\theta_2} \ell_2(G_2(x),y; 
    \boldsymbol\theta_{2})),\tau_{2}\right]$
    \ELSE
    \STATE $\boldsymbol\theta_2 \leftarrow \mathcal{Z}_2\left[\mathbb{E}_{(x,y)}(\nabla_{\boldsymbol\theta_2} \ell_2(x,y; 
    \boldsymbol\theta_{2})),\tau_{2}\right]$
    \ENDIF
    \STATE $\boldsymbol\theta_g \leftarrow \alpha^{\prime}\boldsymbol\theta_g + (1-\alpha^{\prime})(\gamma_1\boldsymbol\theta_1 + (1-\gamma_1)\boldsymbol\theta_2)$
    \IF {$t\geq t^{\prime}$ and $t\ \operatorname{mod} c==0$}
      \STATE $\boldsymbol\theta_1, \boldsymbol\theta_2 \leftarrow \boldsymbol\theta_g$
    \ENDIF
   \ENDFOR
   \STATE \textbf{Return} Parameters of the global learner $\boldsymbol\theta_g$
   \end{algorithmic}
\end{algorithm}
\end{minipage}
\hfill
\begin{minipage}[t]{0.48\linewidth}
\begin{algorithm}[H]
\footnotesize
   \caption{\footnotesize{Generalist-T: The triple version of Generalist for leveraging learning trajectory with respect to three base learners to alleviate both trade-off problems.}}
   \label{alg:Generalist-T}
   \algsetup{linenosize=\tiny}
   \begin{algorithmic}
   \STATE {\bfseries Input:} A DNN classifier $f(\cdot)$ with initial learnable parameters $\boldsymbol\theta_g$ for the global learner and parameters $\boldsymbol\theta_1, \boldsymbol\theta_2, \boldsymbol\theta_3$ for each base learner with objective functions $\ell_1, \ell_2, \ell_3$, learning rates $\tau_1, \tau_2, \tau_3$, optimizers $\mathcal{Z}_1, \mathcal{Z}_2, \mathcal{Z}_3$; functions for the generation of adversarial samples $G_\infty, G_2$; number of iterations $T$; data distribution $\mathcal{D}$; exponential decay rates for ensembling $\alpha^{\prime}=0.999$; mixing ratio $\gamma_1$, $\gamma_2$; starting point and frequency of communication $t^{\prime}, c$. 
   \STATE Initialize $\boldsymbol\theta_g, \boldsymbol\theta_1, \boldsymbol\theta_2, \boldsymbol\theta_3$ in $\Theta$ space.
   \FOR {t $ \leftarrow 1, 2, \cdots , T$}
    \STATE Sample a minibatch $(x, y)$ from the data distribution $\mathcal{D}$.
    \STATE (Optional) Performing model ensembling, data augmentation or label smoothing, etc.
    \STATE $\boldsymbol\theta_1 \leftarrow \mathcal{Z}_1\left[\mathbb{E}_{(x,y)}(\nabla_{\boldsymbol\theta_1} \ell_1(G_\infty(x),y; \boldsymbol\theta_{1})),\tau_1\right]$
    \STATE (Optional) Performing model ensembling, data augmentation or label smoothing, etc.
    \STATE $\boldsymbol\theta_2 \leftarrow \mathcal{Z}_2\left[\mathbb{E}_{(x,y)}(\nabla_{\boldsymbol\theta_2} \ell_2(x,y; \boldsymbol\theta_{2})),\tau_2\right]$
    \STATE (Optional) Performing model ensembling, data augmentation or label smoothing, etc.
    \STATE $\boldsymbol\theta_3 \leftarrow \mathcal{Z}_3\left[\mathbb{E}_{(x,y)}(\nabla_{\boldsymbol\theta_3} \ell_3(G_2(x),y; \boldsymbol\theta_{3})),\tau_3\right]$
    \STATE $\boldsymbol\theta_g \leftarrow \alpha^{\prime}\boldsymbol\theta_g + (1-\alpha^{\prime})(\gamma_1\boldsymbol\theta_1 + \gamma_2\boldsymbol\theta_2+(1-\gamma_1-\gamma_2)\boldsymbol\theta_3)$
    \IF {$t\geq t^{\prime}$ and $t\ \operatorname{mod} c==0$}
      \STATE $\boldsymbol\theta_1, \boldsymbol\theta_2, \boldsymbol\theta_3 \leftarrow \boldsymbol\theta_g$
    \ENDIF
   \ENDFOR
   \STATE \textbf{Return} Parameters of the global learner $\boldsymbol\theta_g$
   \end{algorithmic}
\end{algorithm}
\end{minipage}
\end{figure*}

\subsection{Global Learner Aggregation}
\label{sec:global}
At the early stages of training, base learners are insufficiently trained and thus less reliable. Directly mixing their parameters at this point may mislead optimization and accumulate bias. To address this, we reserve the first $t^{\prime}$ epochs for independently training the base learners. During this warm-up phase, the global learner is updated only through aggregation of their optimization trajectories using an exponential moving average (EMA): 
\begin{equation}
\boldsymbol\theta_g \leftarrow \alpha^{\prime}\boldsymbol\theta_g + (1-\alpha^{\prime})(\sum\limits_{\mathcal{W}=1}^{|\mathcal{A}|-1}\gamma_\mathcal{W}\boldsymbol\theta_\mathcal{W}+(1-\sum\limits_{\mathcal{W}=1}^{|\mathcal{A}|-1}\gamma_\mathcal{W})\boldsymbol\theta_{|\mathcal{A}|}).
\end{equation}
where $\alpha^{\prime}$ is the EMA decay rate and $\gamma_\mathcal{W}$ ($0<\gamma_\mathcal{W}<1$, $\sum\limits_{\mathcal{W}=1}^{|\mathcal{A}|-1}\gamma_\mathcal{W}<1$) denotes the mixing weight of the base learners.

Once the base learners become sufficiently specialized, the global learner periodically redistributes its aggregated parameters back to them every $c$ epochs, serving as a shared initialization that accelerates convergence and improves generalization:
\begin{equation}
\label{eqn:opt_interlude}
\boldsymbol\theta_{\mathcal{W}}^{\star}=\underset{\boldsymbol\theta}{\operatorname{argmin}}{\mathcal{Z}_{\mathcal{W}}^{c}\left[\mathbb{E}_{\mathcal{W}}(\nabla_{\boldsymbol\theta} \ell_{\mathcal{W}}(\mathcal{D}_{\mathcal{W}}; \boldsymbol\theta_{g})),\tau_{\mathcal{W}}\right]}.
\end{equation}
Note that $\boldsymbol\theta_g$ contains part of parameters from each base learners $\boldsymbol\theta_\mathcal{W}$, meaning that there always exists a term updated by gradient information of distribution different from the current subproblem. This mechanism enables fast learning within a given assignment and improves generalization, and the acceleration is applicable to the given assignment for its corresponding base learner only (proof in Appendix \ref{app:theory}).

With all discussed above, the learning progress of Generalist can be constructed by decending the gradient of each base learner $\boldsymbol\theta_\mathcal{W}$ and mixing all of them. The key calculating steps can be summarized in the following equation:
\begin{equation}
\label{eqn:overall}
\small
\left\{
\begin{aligned}
\boldsymbol\theta_{\mathcal{W}}^{t} = &
\mathcal{Z}_{\mathcal{W}}\left[\mathbb{E}_{(x, y) \sim \mathcal{D}_\mathcal{W}}
\big(\nabla_{\boldsymbol\theta_\mathcal{W}} \ell_\mathcal{W}(x, y; \boldsymbol\theta_{\mathcal{W}}^{t-1})\big), \tau_\mathcal{W}\right] \\
&(\mathcal{W}=1,2,\cdots,|\mathcal{A}|)\\
 \\
\boldsymbol\theta_{g}^{t} = & 
\alpha^{\prime}\boldsymbol\theta_g^{t-1} + (1-\alpha^{\prime})(\sum\limits_{\mathcal{W}=1}^{|\mathcal{A}|-1}\gamma_\mathcal{W}\boldsymbol\theta_\mathcal{W}^t+(1-\sum\limits_{\mathcal{W}=1}^{|\mathcal{A}|-1}\gamma_\mathcal{W})\boldsymbol\theta_{|\mathcal{A}|}^t) \\
 \\
\boldsymbol\theta_{\mathcal{W}}^{t} = &
\mathcal{B}(t, t^{\prime}, c)\boldsymbol\theta_{g}^{t}
+ (1-\mathcal{B}(t, t^{\prime}, c))\boldsymbol\theta_{\mathcal{W}}^{t} \\
&(\mathcal{W}=1, 2,\cdots,|\mathcal{A}|)\\
\end{aligned}
\right.
\end{equation}
where $\mathcal{B}(t, t^{\prime}, c)$ is a Boolean function returning 1 if $t\geq t^{\prime}$ and $t\ \operatorname{mod} c==0$, and 0 otherwise. 
\vspace{-5pt}

\subsection{Theoretical Analysis}
\label{sec:theore}

In this section, we theoretically analyze why the decoupled-and-aggregated framework of Generalist can perform well in multiple tasks from two different perspectives. First, from a \textbf{generalization} viewpoint, we show that the population risk of the global learner is controlled by the sum of task-wise regrets of the base learners. Second, from a \textbf{stability} viewpoint, we formalize the insensitivity of a learning algorithm to perturbations in the training data as stablity, and prove that the stability of the global learner can be well controlled by the convex combination of its base learners. These two findings provide a solid theoretical guarantee for the practicality and scalability of Generalist. (Proofs in Appendix \ref{app:theory})

\begin{definition}
(\textbf{Trade-off Regret with Mixed Strategies}) For the natural training assignment or adversarial training assignments $a_1,a_2,\cdots,a_{|\mathcal{A}|}$, consider an algorithm that generates the trajectory of states $\boldsymbol\theta_1,\boldsymbol\theta_2,\cdots,\boldsymbol\theta_{|\mathcal{A}|}$ for ${|\mathcal{A}|}$ base learners, then the regret of ${|\mathcal{A}|}$ base learners on their respective loss function $\ell_1,\ell_2,\cdots,\ell_{|\mathcal{A}|}$ is defined as:
\begin{equation}
\mathbf{R}_{T}=\frac{1}{|\mathcal{A}|} \sum_{a=1}^{|\mathcal{A}|}\left(\sum_{t=1}^{T} \ell_{a}\left(\boldsymbol\theta_a^t\right)-\inf _{\boldsymbol\theta_a^t \in \Theta} \sum_{t=1}^{T} \ell_a\left(\boldsymbol\theta_{a}^{t}\right)\right).
\end{equation}
\end{definition}
Here, the second term corresponds to the oracle state $\boldsymbol\theta_{a}^{\star}$, \textit{i.e.}, the theoretically optimal parameters for each task $a$. Thus, $\mathbf{R}_{T}$ is the sum of the difference between the parameters of each base learner and its theoretically optimal parameters for each task.

Based on the definition, we establish the following upper bound on the expected error of the classifier trained by Generalist with respect to $\mathbf{R}_{T}$ as:
\begin{theorem}
\label{theorem:Theorem-1}
 Consider an algorithm with regret bound $R_{T}$ that generates the trajectory of states for $|\mathcal{A}|$ base learners. For any parameter state $\boldsymbol\theta \in \Theta$, given a sequence of convex surrogate evaluation functions ${\ell: \Theta\mapsto [0, 1]_{a\in \mathcal{A}}}$ drawn i.i.d. from some distribution $\mathcal{L}$, the expected error of the global learner $\boldsymbol\theta_{g}$ on all tasks over the test set can be bounded with probability at least $1-\delta$ as:
\begin{equation}
\underset{\ell \sim \mathcal{L}}{\mathbb{E}} \ell\left(\boldsymbol\theta_{g}\right) \leq \underset{\ell \sim \mathcal{L}}{\mathbb{E}} \ell\left(\boldsymbol\theta\right)+\frac{\mathbf{R}_{T}}{T}+2\sqrt{\frac{2}{T}\log \frac{1}{\delta}}.
\end{equation}
\end{theorem}

This result shows that any strategy that reduces the task-specific regret $\mathbf{R}_{T}$ will also tighten the error bound of the global learner. Considering Generalist divides the trade-off problem into several independent tasks, Theorem \ref{theorem:Theorem-1} guarantees that lowering the error of individual tasks directly lowers the risk bound of the global learner. In practice, we can apply customized learning rate strategies, optimizers, and weight averaging to guarantee the error reduction of each base learners.

We next analyze the sensitivity of the Generalist against perturbations in the training data. To this end, we introduce the notion of \emph{$\epsilon$-stability}, a variant of uniform stability in~\cite{shalev2010learnability}. 
In the general learning setting, it identifies algorithmic stability---the insensitivity of the learned predictor to replacing one training example---as the key \emph{necessary and sufficient} condition for statistical learnability. 
\begin{definition}
\noindent\textbf{(\(\epsilon\)-Stability)} 
A learning algorithm admits \(\epsilon\)-stability in the sense that, for any two training sets
$\mathcal{D},\mathcal{D}^{\prime}$ differing in exactly one example and any test point $z$ in the test set $\mathcal{T}$,
\begin{equation}
\begin{aligned}
\big|\ell\!\left(f_{\theta(\mathcal{D})},z\right)-\ell\!\left(f_{\theta(\mathcal{D}^{\prime})},z\right)\big| \;\le\; \epsilon,  
\end{aligned}
\end{equation}
where $\theta(\mathcal{D})$, $\theta(\mathcal{D}^{\prime})$ are parameters learned by the algorithm, and $f_{\theta(\mathcal{D})}$, $f_{\theta(\mathcal{D}^{\prime})}$  are the corresponding predictors.
\end{definition}

Intuitively, $\epsilon$-stability controls how much the loss of the returned predictor can change when the training data is perturbed at a single point. This directly yields \emph{distribution-free generalization} guarantees and explicitly isolates the contribution of the learning rule (rather than the hypothesis class complexity). 
Adopting this concept allows us to quantify how Generalist reacts to sample-level randomness during training.

\begin{theorem}
\noindent\textbf{(Global Stability)} 
\label{thm:2}
Assume each base learner reaches $\epsilon_a$-stablity on its own task, for $a=1, 2,\cdots,|\mathcal{A}|$, and let $\bar\theta$ denote the previous global iterate (i.e., the global parameter before the current aggregation round). 
Then, the global learner $f_{\theta_g}$ produced by Generalist framework at the current round admits the stability bound
\begin{equation}
\begin{aligned}
\epsilon_g \;\le\; \epsilon_\oplus \;+\; C\,\sum_{a=1}^{|\mathcal{A}|}\gamma_a\,\|\theta_a-\bar\theta\|^2,
\end{aligned}
\end{equation}
where $\epsilon_\oplus:=\sum_{a=1}^{|\mathcal{A}|}\gamma_a\,\epsilon_a$ and $C$ is a bounded constant.
\end{theorem}

Theorem \ref{thm:2} shows that the global learner’s instability $\epsilon_g$ is \emph{not} amplified by aggregated training. Instead, it is controlled by the convex combination of per-task instabilities, along with a small geometric term that quantifies how closely the base parameters cluster around the previous global iterate. This result highlights that the decoupled training paradigm of Generalist mitigates, rather than amplifies, per-task variability. Consequently, stochastic noise arising from mini-batch sampling and task heterogeneity is effectively averaged out during aggregation, leading to more stable optimization and better balanced performance across tasks.

\renewcommand{\arraystretch}{0.85}
\vspace{-5pt}
\section{Experiments}
\label{sec:exp}

In this section, we conduct extensive experiments to evaluate the effectiveness and generality of the proposed Generalist framework. The evaluation covers three aspects: standard adversarial robustness on CIFAR-10 and CIFAR-100, scalability to large-scale datasets using ImageNet, and generalization to out-of-distribution (OOD) perturbations. As described in Section~\ref{sec:method}, Generalist has two variants depending on the number of base learners: \textbf{Generalist-D} (double base learners) and \textbf{Generalist-T} (triple base learners). Generalist-D can be further instantiated as Generalist-D ($NT+\ell_\infty$) to address the natural–robustness trade-off, or Generalist-D ($\ell_\infty+\ell_2$) to address robustness across different norm constraints. We compare these variants against a wide range of state-of-the-art adversarial training baselines under unified settings.

\vspace{-5pt}
\subsection{Setup}
\label{apd:exp}
\textbf{Baselines.} In addition to vanilla AT using PGD \cite{PGD}, we compare against two groups of baselines. The first group focuses on improving natural generalization of AT, including: AT with half-half loss (averaging natural and adversarial losses)\cite{DBLP:journals/corr/GoodfellowSS14}, TRADES with $\beta=1$\cite{trades}, Friendly Adversarial Training (FAT)\cite{fat}, Interpolated Adversarial Training (IAT)\cite{IAT}, Self-Consistent Robust Error (SCORE)\cite{LSE}, Adaptive Gradient Reconstruction (AGR)\cite{AGR}, and Pixel-reweighted Adversarial Training (PART)\cite{PART}. The second group addresses robustness across different norm budgets, including: AT with averaged losses over perturbations\cite{tramer2019adversarial}, Multi Steepest Descent (MSD)\cite{MSD}, Extreme-norm Adversarial Training (E-AT)\cite{E_AT}, and Robust Method against Multiple Perturbations (RAMP)~\cite{RMC}. All models are trained from scratch using the publicly available code of each method.

\textbf{Evaluation.} To evaluate robustness, we apply adversarial attacks including 20-step PGD \cite{PGD}, \textit{i.e.}, PGD$^{20}$, and AutoAttack (AA) \cite{DBLP:conf/icml/Croce020a} that is an ensemble of four attacks (\textit{i.e.}, two types of APGD attacks, FAB and Square attack) and widely regarded as the most reliable attacks in adversarial robustness. Subscripts distinguish norms used for attacks, e.g., AA$_\infty$ and AA$_2$. We also report {union robustness} (Union), defined as the average of AA$_\infty$ and AA$_2$, to reflect robustness under multiple perturbation types.

\begin{table*}[!t]
\renewcommand{\arraystretch}{1.1}
\centering
\caption{Comparison (\%) of Generalist with different training methods using ResNet-18 and WRN-32-10 on CIFAR-10. The attack budgets are set to $\varepsilon=8/255$ for the $\ell_\infty$ norm and $\varepsilon=128/255$ for the $\ell_2$ norm. The best and second-best results are highlighted in \textbf{bold} and \underline{underlined}, respectively. Standard deviations are omitted as they are negligible ($<0.5\%$). 
}
\label{tab:cifar10}
\resizebox{1.0\linewidth}{!}{
\begin{tabular}{c|ccccccccccccccccc}
\bottomrule
\multirow{3}{*}{Method} &           \multicolumn{6}{c}{ResNet-18}                            & &\multicolumn{6}{c}{WRN-32-10}                       \\ \cmidrule{2-7} \cmidrule{9-14}
& Natural&$\text{PGD}_\infty^{20}$& $\text{AA}_{\infty}$ &$\text{PGD}_2^{20}$ & $\text{AA}_2$  &Union&& Natural&$\text{PGD}_\infty^{20}$& $\text{AA}_{\infty}$ &$\text{PGD}_2^{20}$ & $\text{AA}_2$&Union\\ 
\midrule
AT~\cite{PGD}& 84.32 & 48.29& 44.37 & 61.07   & 56.99&50.68&&87.32&49.01&46.11&57.00&53.97&50.04\\
AT (NT+$\ell_\infty$)~\cite{halfat}&87.34&44.51&40.06&57.79&54.97&47.52&&89.27&48.95&44.81&58.55&55.46&50.14\\
TRADES ($\beta=1$)~\cite{trades}&87.88&45.58&40.32&62.15&58.01&49.17&&87.20&51.33&49.81&60.01&56.70&53.26\\
FAT~\cite{fat}& 87.72 & 46.69 & 43.14& 58.64  & 55.83&49.49&&89.65&48.74&44.73&57.90&53.54&49.14 \\
IAT~\cite{IAT}& 83.70 & 40.83 & 35.13& 55.16  & 51.39&43.26&&88.04&48.85&42.23&64.95&59.66&50.95 \\
SCORE~\cite{LSE}&87.72&42.73&32.70&63.25&55.65&44.18&&88.48&47.11&38.86&64.62&57.80&48.33 \\
AGR~\cite{AGR}&85.18&48.80&44.42&60.07&57.29&50.81&&87.09&50.98&48.19&59.16&56.24&52.22\\
PART~\cite{PART}&83.07&42.98&41.22&59.98&57.57&49.40&&84.29&42.08&41.13&59.20&57.13&49.13  \\
\midrule
AT ($\ell_\infty$+$\ell_2$)~\cite{tramer2019adversarial}&85.61&43.35&39.48& 63.16&61.34&50.41&&87.55&49.17&45.62&64.36&63.01&54.32\\
MSD~\cite{MSD}&82.91&\underline{50.52} & 46.08& 62.73 & 58.97&52.53&&86.27&51.07&46.65&\underline{69.66}&\underline{67.12}&56.89\\ 
E-AT~\cite{E_AT}&72.68&39.38&34.98&57.84&55.70&45.34&&71.75&38.98&34.77&57.32&55.00&44.89 \\
RMC~\cite{RMC}&82.00&\textbf{52.26}&\textbf{48.32}&58.91&55.57&51.95&&80.18&53.87&50.00&61.62&58.92&54.46 \\

\midrule
Generalist-D ($NT+\ell_\infty)$  &\textbf{89.09} & 50.01& 46.07  & 62.08  & 58.11&52.09&&\textbf{91.03}&\underline{56.88}& \underline{52.91}&63.96&58.95 &55.93 \\
Generalist-D  ($\ell_\infty+\ell_2$) & 86.94 &50.46 &\underline{46.24} & \textbf{67.63}  & \textbf{65.09}&\textbf{55.67}&&88.10&\textbf{57.38}&\textbf{53.29}&\textbf{70.85}&\textbf{68.07}&\textbf{60.68}\\
Generalist-T ($NT+\ell_\infty+\ell_2$)  &\underline{88.03} & 47.61 & 43.23 &\underline{65.89}&\underline{63.40}&\underline{53.32}&&\underline{89.66}&54.00&50.62&66.39&63.44&\underline{57.03}\\ 
\bottomrule
\end{tabular}}
\end{table*}

\begin{table*}[!t]
\renewcommand{\arraystretch}{1.1}
\centering
\caption{Comparison (\%) of Generalist with different training methods using ResNet-18 and WRN-32-10 on CIFAR-100. The attack budgets are set to $\varepsilon=8/255$ for the $\ell_\infty$ norm and $\varepsilon=128/255$ for the $\ell_2$ norm. The best and second-best results are highlighted in \textbf{bold} and \underline{underlined}, respectively. Standard deviations are omitted as they are negligible ($<0.5\%$). 
}
\label{tab:cifar100}
\resizebox{1.0\linewidth}{!}{
\begin{tabular}{c|ccccccccccccccccc}
\bottomrule
\multirow{3}{*}{Method} &           \multicolumn{6}{c}{ResNet-18}                            & &\multicolumn{6}{c}{WRN-32-10}                       \\ \cmidrule{2-7} \cmidrule{9-14}
& Natural&$\text{PGD}_\infty^{20}$& $\text{AA}_{\infty}$ &$\text{PGD}_2^{20}$ & $\text{AA}_2$  &Union&& Natural&$\text{PGD}_\infty^{20}$& $\text{AA}_{\infty}$ &$\text{PGD}_2^{20}$ & $\text{AA}_2$&Union\\ 
\midrule
AT~\cite{PGD}&60.10&28.22&23.87&30.08&26.87&25.37&&57.74&29.07&25.64&35.73&31.50&28.57 \\
AT (NT+$\ell_\infty$)~\cite{halfat}&60.87&22.64&19.17&28.96&25.93&22.55&&63.75&26.11&22.98&35.22&21.20&22.09\\
TRADES ($\beta=1$)~\cite{trades}&60.18&28.93&23.22&34.32&31.21&27.22&&61.47&24.35&21.63&34.42&31.50&26.57\\
FAT~\cite{fat}&61.71&22.93&20.01&32.56&30.50&25.26&&65.30&24.03&21.38&32.67&29.91&25.65\\
IAT~\cite{IAT}&57.04&21.40&15.50&55.76&28.73&22.12&&63.21&23.16&18.89&35.46&31.35&25.12 \\
SCORE~\cite{LSE}&44.27&27.84&23.36&32.57&27.99&25.68&&39.65&27.06&22.55&29.18&24.35&23.45 \\
AGR~\cite{AGR}&58.25&23.86&20.85&34.02&31.06&25.96&&62.42&27.10&24.29&34.06&21.04&22.67\\
PART~\cite{PART}&56.42&20.45&18.04&31.68&29.28&23.66&&57.39&21.11&19.18&31.78&29.82&24.50 \\
\midrule
AT ($\ell_\infty$+$\ell_2$)~\cite{tramer2019adversarial}&56.36&19.62&16.82&35.43&33.22&25.02&&58.70&25.17&22.19&37.72&35.42&28.81\\
MSD~\cite{MSD}&58.30&28.23&23.58&38.27&34.38&28.98&&62.51&26.78&23.54&37.12&33.74&28.64\\ 
E-AT~\cite{E_AT}&45.48&19.73&15.94&32.91&30.06&23.00&&44.07&18.97&15.33&31.85&29.04&22.19\\
RMC~\cite{RMC}&55.48&25.73&22.29&29.76&26.71&24.50&&56.53&29.90&25.65&36.55&32.76&29.21\\

\midrule
Generalist-D ($NT+\ell_\infty)$  &\textbf{62.97}&\textbf{29.48}&\underline{23.96}&39.14&34.23&29.10&&\textbf{66.66}&\underline{30.47}&\underline{26.86}&38.67&34.04&30.45\\
Generalist-D  ($\ell_\infty+\ell_2$) &60.90&\underline{29.43}&\textbf{24.23}&\textbf{42.37}&\textbf{38.25}&\textbf{33.84}&&64.85&\textbf{30.65}&\textbf{27.29}&\textbf{42.63}&\textbf{39.47}&\textbf{33.38}\\
Generalist-T ($NT+\ell_\infty+\ell_2$)  &\underline{62.68}&28.94&23.88&\underline{41.07}&\underline{36.34}&\underline{30.11}&&\underline{66.47}&29.61&26.23&\underline{41.18}&\underline{37.77}&\underline{32.00}\\ 
\bottomrule
\end{tabular}}
\vspace{-5pt}
\end{table*}

\subsection{Performance on Standard Benchmarks}
\label{sec:benchmark}

To evaluate the effectiveness of Generalist under standard benchmark settings, we conduct experiments with ResNet-18~\cite{DBLP:conf/cvpr/HeZRS16} and WRN-32-10~\cite{DBLP:journals/corr/ZagoruykoK16} on CIFAR-10~\cite{cifar} and CIFAR-100 \cite{cifar}. We train all models with SGD using momentum $0.9$ for $120$ epochs. The weight decay factor is $3.5\times10^{-3}$ for ResNet-18 and $7\times10^{-4}$ for WRN-32-10. For adversarial-training base learners, the initial learning rate is set to $0.01$ for ResNet-18 and $0.1$ for WRN-32-10 until epoch 40, after which it decays linearly. Following the settings in previous studies~\cite{crocerobustbench}, we set the perturbation budgets $\epsilon$ to $8/255$ for the $\ell_\infty$ norm and $128/255$ for the $\ell_2$ norm. The inner maximization employs PGD with 10 steps and step size $\epsilon/4$. For natural-training base learners, the initial learning rate is $0.1$ with weight decay $5\times10^{-4}$ for both architectures. For Generalist, we set $t^{\prime}=75$.

As shown in Tables \ref{tab:cifar10} and \ref{tab:cifar100}, on both CIFAR-10 and CIFAR-100, we first observe that Generalist-D achieves outstanding performance in alleviating either the natural–robustness trade-off or the robustness trade-off across norms. For example, Generalist-D ($NT+\ell_\infty$) consistently improves natural accuracy over existing robust training methods while maintaining comparable robustness. On CIFAR-10 with ResNet-18, it is the only method to achieve natural accuracy above 89\%, whereas the best competing method, TRADES, reaches only 87.88\%. In terms of robustness, Generalist-D ($NT+\ell_\infty$) attains 46.07\% under AA$_\infty$, substantially higher than TRADES (40.32\%). Similarly, Generalist-D ($\ell_\infty+\ell_2$) consistently achieves the best union robustness across all datasets and architectures. For instance, on CIFAR-100 with WRN-32-10, it improves union robustness to 33.38\%, surpassing the best baseline by more than 4\%. These results highlight the effectiveness of Generalist-D when focusing on a single trade-off issue.

When both trade-offs are expected to be mitigated simultaneously, Generalist-T provides a strong solution. Although in almost all cases, it is left behind by Generalist-D, which is more focused, Generalist-T still exceeds the performance of current methods in each aspect. For example, on CIFAR-100 dataset, when comparing Generalist-T with baselines designed for the natural-robustness trade-off mitigation (the first group of baselines in Table  \ref{tab:cifar100}), we observe that Generalist-T obtains higher natural accuracy than the existing best result (66.47\% vs 65.30\%, +1.17\%) on WRN-32-10. Meanwhile, since Generalist-T learns knowledge from a base learner that is adversarially trained under $\ell_2$ norm, its robustness against $\ell_2$ attacks increases markedly, raising AA$_2$ from 31.50\% to 37.77\%. Similarly, when comparing with methods aiming at universal robustness (the second group of baselines in Table  \ref{tab:cifar100}), we see that Generalist-T not only achieves higher union robustness but also boosts natural accuracy. The above evidences demonstrate the superior performance of Generalist-T in mitigating both trade-off issues.

It is worth noting that the final Generalist models are the same size as those trained by baseline methods. Moreover, Generalist-D and Generalist-T are trained using only the standard cross-entropy loss, without resorting to advanced loss designs. This simplicity indicates that further improvements may be achievable with more sophisticated objectives, suggesting promising potential for future extensions.

\begin{table*}[!t]
\renewcommand{\arraystretch}{1.1}
\centering
\caption{Comparison (\%) of Generalist with different training methods using ResNet-50 and WRN-50-2 on ImageNet. The attack budgets are set to $\varepsilon=4/255$ for the $\ell_\infty$ norm and $\varepsilon=64/255$ for the $\ell_2$ norm. Following \cite{crocerobustbench}, the evaluation is performed on 5000 images randomly sampling from the validation set. The best and second-best results are highlighted in \textbf{bold} and \underline{underlined}, respectively. Standard deviations are omitted as they are negligible ($<0.5\%$).
}
\label{tab:imagenet}
\resizebox{1.0\linewidth}{!}{
\begin{tabular}{c|ccccccccccccccccc}
\bottomrule
\multirow{3}{*}{Method}&           \multicolumn{6}{c}{ResNet-50}                            & &\multicolumn{6}{c}{WRN-50-2}                       \\ \cmidrule{2-7} \cmidrule{9-14}
& Natural&$\text{PGD}_\infty^{20}$& $\text{AA}_{\infty}$ &$\text{PGD}_2^{20}$ & $\text{AA}_2$  &Union&& Natural&$\text{PGD}_\infty^{20}$& $\text{AA}_{\infty}$ &$\text{PGD}_2^{20}$ & $\text{AA}_2$&Union\\ 
\midrule
Fast-AT~\cite{wong2020fast}&55.62&30.32&26.24&51.48&50.32&38.28&&58.48&31.56&28.10&51.62&49.90&39.00\\
AT~\cite{salman2020adversarially}&64.02&39.20&34.96&58.82&57.30&46.13&&68.46&41.30&38.14&63.42&62.30&50.22\\
\midrule
Generalist-D ($NT+\ell_\infty)$&\textbf{65.48}&\textbf{40.52}&35.88&59.94&58.66&47.27&&\textbf{68.92}&\underline{43.30}&39.60&63.28&62.10&50.85\\
Generalist-D  ($\ell_\infty+\ell_2$) &64.92&39.88&35.86&\underline{60.08}&\textbf{58.98}&\textbf{47.42}&& 68.36&\textbf{43.38}&39.76&\textbf{63.60}&\textbf{62.64}&\textbf{51.25}\\
Generalist-T ($NT+\ell_\infty+\ell_2$)&\underline{65.38}&\underline{40.16}&\textbf{35.88}&\textbf{60.28}&\underline{58.94}&\underline{47.41}&&\underline{68.70}&43.16&\textbf{39.76}&\underline{63.54}&\underline{62.58}&\underline{51.17}\\ 
\bottomrule
\end{tabular}}
\vspace{-10pt}
\end{table*}

\begin{table*}[!t]
\renewcommand{\arraystretch}{1.1}
\centering
\caption{Generalization (\%) of AT methods on out-of-distribution (OOD) datasets. We highlight the best result in \textbf{bold} and the second-best results with \underline{underlines}. Standard deviations are omitted as they are negligible ($<0.5\%$).}
\label{tab:ood}
\resizebox{0.85\linewidth}{!}{
\begin{tabular}{c|ccccccccccccccccc}
\bottomrule
\multirow{3}{*}{Method}&           \multicolumn{2}{c}{CIFAR-10-C}   &   &\multicolumn{2}{c}{CIFAR-10-P}&&\multicolumn{2}{c}{CIFAR-100-C} & &\multicolumn{2}{c}{CIFAR-100-P}               \\ \cmidrule{2-3} \cmidrule{5-6}\cmidrule{8-9}\cmidrule{11-12}
&RN18&WRN32& & RN18&WRN32&&RN18&WRN32& & RN18&WRN32\\ 
\midrule
AT~\cite{PGD}&75.20&74.83&&83.35&83.11&&41.44&42.10&&51.37&54.28\\
AT(NT+$\ell_\infty$)~\cite{halfat}&75.21&76.30&&83.35&84.84&&42.46&46.35&&54.48&60.52\\
TRADES ($\beta=1$)~\cite{trades}&75.12&76.66&&84.54&86.04&&44.18&46.63&&55.65&58.34\\
FAT~\cite{fat}&75.94&77.25&&85.10&86.13&&45.45&47.55&&57.83&61.54\\
IAT~\cite{IAT}&72.58&77.39&&81.39&86.51&&42.24&47.67&&54.26&60.07\\
SCORE~\cite{LSE}&74.04&76.83&&85.02&85.95&&19.58&32.36&&42.48&38.02\\
AGR~\cite{AGR}&74.05&74.79&&83.98&84.54&&43.87&46.32&&55.25&59.02\\
PART~\cite{PART}&71.86&70.09&&80.78&79.13&&42.09&43.32&&53.20&54.32\\
\midrule
AT ($\ell_\infty$+$\ell_2$)~\cite{tramer2019adversarial}&74.04&76.21&&81.50&84.13&&43.10&45.07&&53.68&55.73\\
MSD~\cite{MSD}&45.79&48.83&&48.03&85.75&&22.95&23.96&&25.64&26.89\\ 
E-AT~\cite{E_AT}&36.81&51.10&&37.40&55.56&&36.30&34.88&&43.38&41.99\\
RMC~\cite{RMC}&53.94&56.46&&57.86&60.08&&23.56&23.11&&26.41&25.83\\

\midrule
Generalist-D ($NT+\ell_\infty)$ &\underline{77.67}&\underline{79.25}&&\underline{85.68}&\underline{87.93}&&\underline{48.51}&\underline{50.60}&&\underline{60.68}&\underline{62.94}\\
Generalist-D  ($\ell_\infty+\ell_2$) &76.41&78.20 &&85.33&86.87&&47.35&50.27&&57.93&61.73\\
Generalist-T ($NT+\ell_\infty+\ell_2$)&\textbf{77.76}&\textbf{79.64}&&\textbf{85.79}&\textbf{87.96}&&\textbf{49.07}&\textbf{51.09}&&\textbf{60.81}&\textbf{62.96}\\ 
\bottomrule
\end{tabular}}
\vspace{-7pt}
\end{table*}

\vspace{-5pt}
\subsection{Performance on Large-scale Dataset}
\label{sec:largescale}

To assess whether Generalist scales to realistic scenarios, we further evaluate it on ImageNet~\cite{deng2009imagenet} using ResNet-50~\cite{DBLP:conf/cvpr/HeZRS16} and WRN-50-2~\cite{DBLP:journals/corr/ZagoruykoK16}. Although adversarial training has been extensively studied on small datasets like the baselines in Tables~\ref{tab:cifar10} and~\ref{tab:cifar100}, its application to large-scale settings remains limited due to the substantial computational cost. Here, we compare Generalist against two representative approaches on large-scale adversarial training: Fast Adversarial Training (Fast-AT)\cite{wong2020fast} and standard AT\cite{salman2020adversarially}. For a fair comparison, robustness is evaluated under an $\ell_\infty$ budget of $4/255$ and an $\ell_2$ budget of $64/255$. To reduce computation, adversarial examples are crafted with 3 PGD steps instead of 10. Given the scale of ImageNet, we set the weight decay to $1\times10^{-4}$ for adversarial training and $2\times10^{-4}$ for natural training, and double the number of epochs to ensure convergence. All other hyperparameters follow section \ref{sec:benchmark}.

As shown in Table \ref{tab:imagenet}, Generalist also achieves strong performances on the ImageNet dataset. For example, considering the natural accuracy, Generalist-D ($NT+\ell_\infty$) increases it from 64.02\% to 65.48\% on ResNet-50 while maintaining high $\ell_\infty$ robustness (+0.92\% AA$_\infty$ over AT). Compared to the specificity of Generalist-D, Generalist-T also achieves remarkable performances in alleviating both trade-off issues: On WRN-50-2, natural accuracy is improved from 68.46\% to 68.70\% while achieving high union robustness (+0.95\%).

\vspace{-5pt}
\subsection{Performance on OOD Datasets}
\label{sec:ood}

\begin{figure}[t]
  \begin{minipage}{1.0\linewidth}
    \makebox[.48\linewidth]{\includegraphics[width=.48\linewidth]{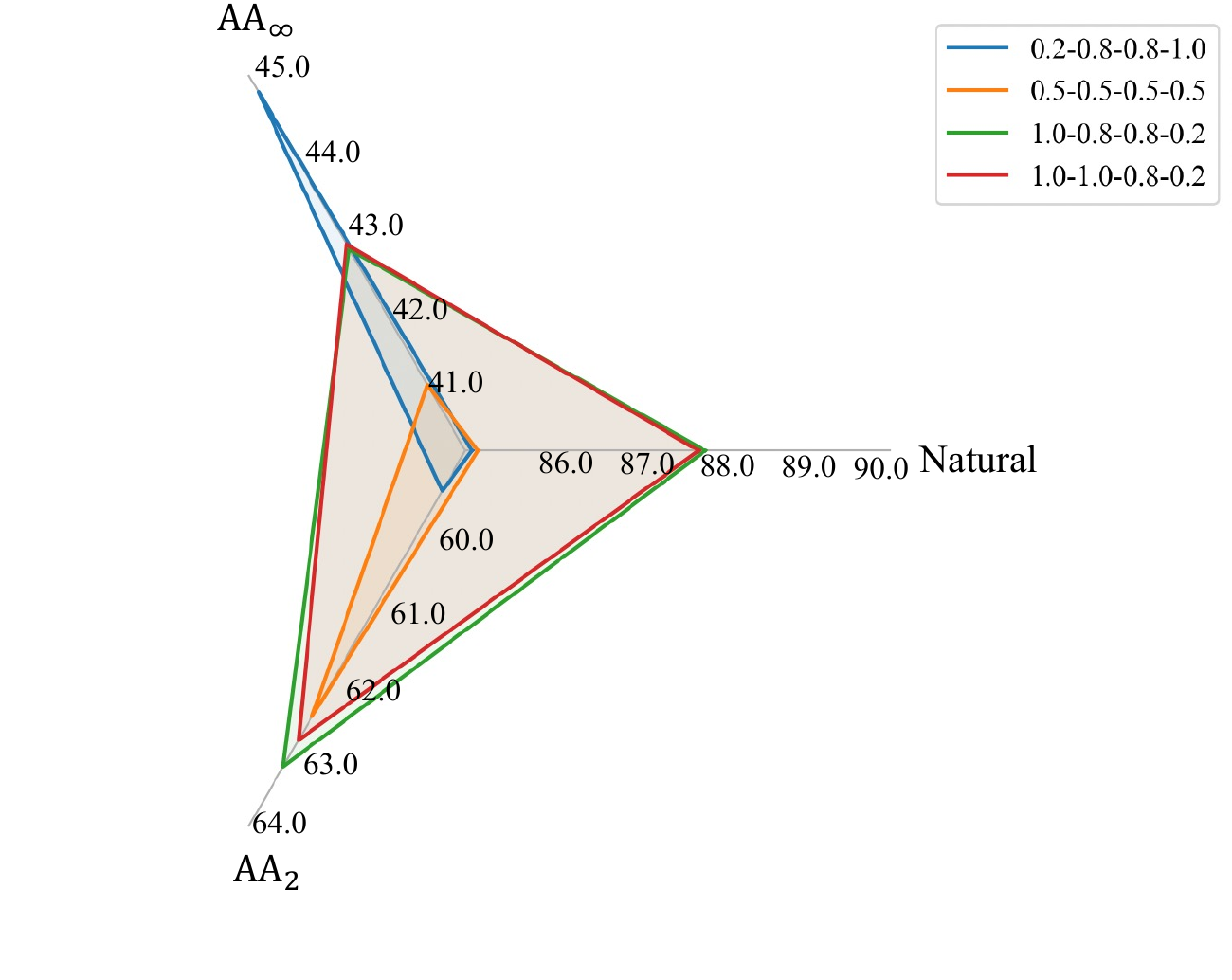}}%
    \hspace{8pt}
    \makebox[.48\linewidth]{\includegraphics[width=.48\linewidth]{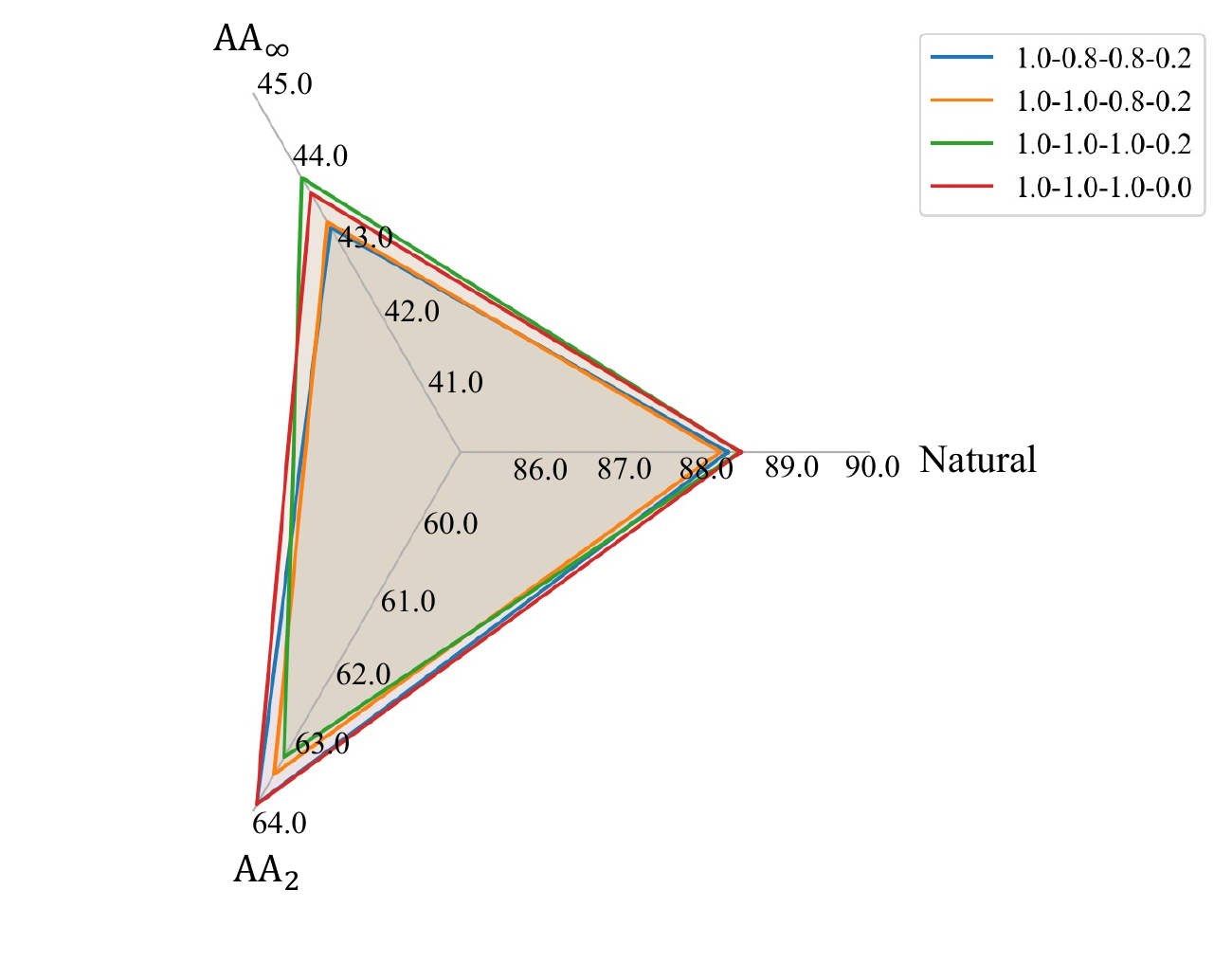}}%
    
    {\makebox[.48\linewidth]{(a) Change Trends}}%
    {\makebox[.48\linewidth]{(b) Decay Stages}}%
  \end{minipage}%
\caption{The impact of $\gamma_1$ to the performances of Generalist-T.}
\vspace{-10pt}
\label{fig:triple_gamma1}
\end{figure}
In real-world deployment, models inevitably encounter not only adversarial perturbations but also unforeseen distribution shifts. Out-of-distribution (OOD) data~\cite{liu2021towards} differ from the training data in aspects such as style, background, or physical distortions (e.g., brightness changes or glass blurring). Unlike in-distribution samples, these OOD inputs can mislead models even without adversarial perturbations. Ideally, a robust model should withstand not only attacks incorporated during training but also generalize to perturbations from previously unseen domains.

To evaluate this property, we test on four OOD benchmarks: CIFAR-10-C, CIFAR-100-C, CIFAR-10-P, and CIFAR-100-P~\cite{hendrycks2019robustness}, where all corruptions are unseen for both AT baselines and Generalist. For CIFAR-10-C and CIFAR-100-C, we report model accuracy under level-5 natural corruptions. For CIFAR-10-P and CIFAR-100-P, we report the average accuracy across corruption sequences. Results are summarized in Table~\ref{tab:ood}.

The findings show that Generalist achieves consistently better resistance to OOD attacks compared with baseline methods, confirming its ability to integrate knowledge from diverse tasks and generalize to unseen scenarios. For example, on the CIFAR-10-C dataset, the accuracy of vanilla AT is 75.20\% while Generalist-T improves it to 77.76\%. In addition, when comparing the performances between Generalist-D and Generalist-T, it is interesting to see that Generalist-T achieves higher performance across all four datasets. This is because a larger number of base learners contributes to a more knowledgeable global learner, which in turn captures invariant features more effectively and enhances robustness against previously unseen threats.

\begin{figure}[t]
  \begin{minipage}{0.48\linewidth}
    \centering

    \includegraphics[width=\linewidth]{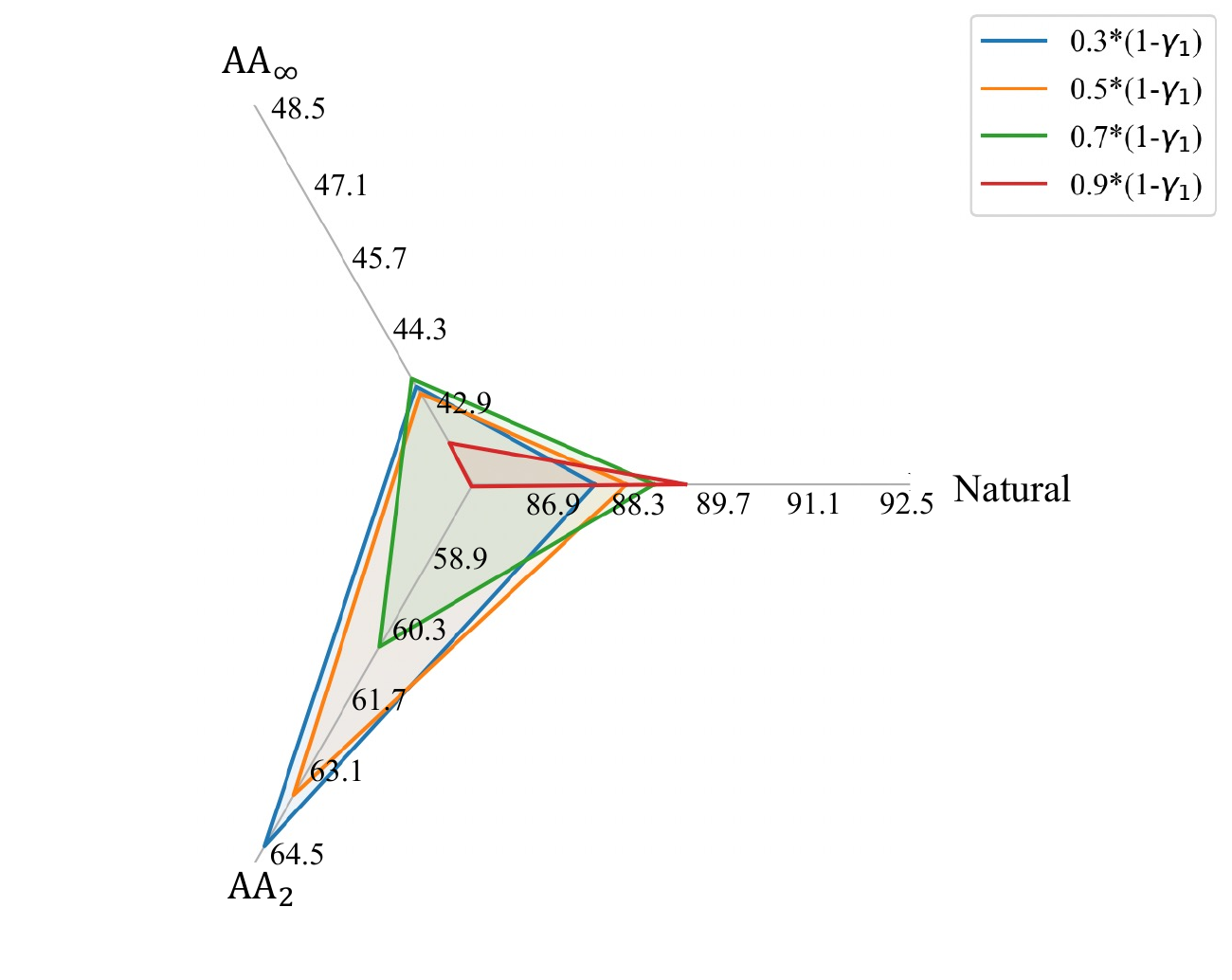}
    \caption{The impact of $\gamma_2$ to the performances of Generalist-T.}
    \label{fig:triple_gamma2}
  \end{minipage}\hfill
  \begin{minipage}{0.48\linewidth}
    \centering
    \includegraphics[width=\linewidth]{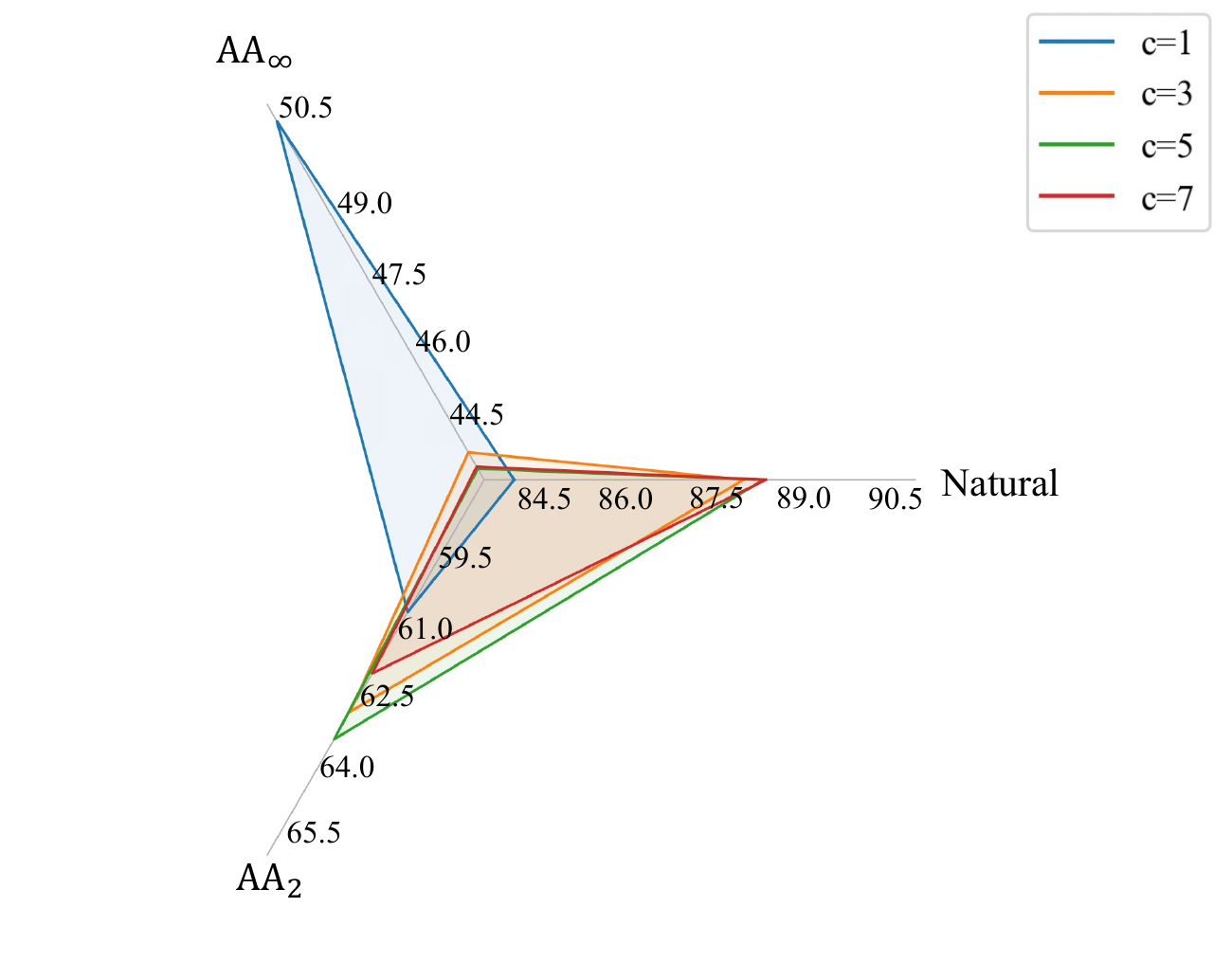}
    \caption{The impact of $c$ to the performances of Generalist-T.}
    \label{fig:frequency}
  \end{minipage}
\end{figure}

\section{Ablation Studies}
\label{sec:mixing}

In this section, we conduct a series of ablation studies to better understand the Generalist framework. As illustrated in Algorithms~\ref{alg:Generalist-D} and~\ref{alg:Generalist-T}, two key factors govern the trade-off behavior of Generalist: the \emph{mixing ratio $\gamma$} (with $\gamma_1$ for Generalist-D and $\gamma_1,\gamma_2$ for Generalist-T) and the \emph{communication frequency $c$} between base learners and the global learner. Unless otherwise stated, all experiments are performed on CIFAR-10 with ResNet-18. For space considerations, we report results on Generalist-T in the main text and defer those of Generalist-D to Appendix \ref{app:abla_d}.

\subsection{Mixing Strategies of $\gamma$} 

In the Generalist framework, the coefficient $\gamma_{\mathcal{W}}$ controls the relative contribution of each base learner to the global parameters. Although $\gamma$ is a scalar, we dynamically adjust its value during training to ensure that parameter aggregation occurs only after all base learners have acquired sufficient task-specific knowledge. Specifically, we define several breakpoints along the training trajectory and update $\gamma$ through a piecewise linear schedule.

For Generalist-T, $\gamma_{\mathcal{W}}$ corresponds to $\gamma_1$ and $\gamma_2$. We first fix $\gamma_2$ and examine the effect of $\gamma_1$, as shown in Figure~\ref{fig:triple_gamma1}. Figure~\ref{fig:triple_gamma1}(a) compares different change trends of $\gamma_1$. We observe that decaying $\gamma_1$ over time yields the most balanced performance across all metrics—achieving not only the best AA$_2$ and natural accuracy but also comparable AA$_\infty$. Figure~\ref{fig:triple_gamma1}(b) further investigates different decay stages, showing that a late-stage decay schedule (1.0–1.0–1.0–0.0) achieves the best results, confirming the benefit of gradual reduction once base learners have stabilized. Consequently, we adopt this configuration as the default setting in experiments. 

Next, we analyze the effect of $\gamma_2$, noting that $\gamma_1+\gamma_2<1$ must hold to preserve a positive contribution from the third base learner $\theta_3$. We therefore introduce a hyperparameter $b$ and set $\gamma_2=b(1-\gamma_1)$. As shown in Figure~\ref{fig:triple_gamma2}, $\gamma_2$ directly governs the trade-off between AA$_2$ robustness and natural accuracy. Empirically, we find that $\gamma_2=0.5(1-\gamma_1)$ provides the most favorable balance between AA$_2$ and natural accuracy, and we also adopt this configuration as the default setting in experiments.

\vspace{-5pt}
\subsection{Communication Frequency $c$} 

In Generalist, the parameter $c$ determines how frequently the global learner communicates with base learners during training. With the mixing ratio fixed, we vary $c$ from 1 to 7 to investigate its effect. The results for Generalist-T are shown in Figure \ref{fig:frequency}.

We observe that when $c$ is too small (\textit{e.g.}, $c=1$), base learners are synchronized too frequently, which prevents them from fully adapting to their respective sub-tasks. This leads to lower natural accuracy and weaker AA$_2$ robustness, even though AA$_\infty$ improves due to the dominance of adversarial signals. As $c$ increases to a moderate value, both natural accuracy and AA$_2$ improve markedly, indicating that allowing base learners sufficient independent optimization steps helps them specialize while still benefiting from periodic aggregation. However, when $c$ becomes too large (\textit{e.g.}, $c=7$), the communication becomes too sparse, and the global learner struggles to integrate knowledge effectively, causing a slight drop. Overall, the results reveal that $c$ implicitly governs the balance between specialization and synchronization among tasks. Setting $c=5$ provides the most favorable trade-off, achieving high natural accuracy and strong robustness across both $\ell_\infty$ and $\ell_2$ perturbations. Therefore, we adopt $c=5$ as the default communication frequency in experiments.

\vspace{-5pt}
\subsection{Transferability of Hyperparameters}

In practice, the mixing parameter $\gamma$ and communication frequency $c$ can be selected without prior knowledge of the target model or dataset. We first identify the optimal configuration on a specific architecture and dataset, and then directly transfer these hyperparameters to other settings. In other words, the best $\gamma$, $c$, and their scheduling strategies found on one model or dataset can be effectively reused for others with minimal or no fine-tuning.
For instance, in the CIFAR-100 experiments reported in Table~\ref{tab:cifar100}, we simply adopt the optimal parameters and update strategies obtained from CIFAR-10, yet still achieve strong performance. A similar observation holds for large-scale datasets such as ImageNet, where using the same transferred parameters yields higher natural accuracy and union robustness than the baselines.

\begin{figure}[t]
  \begin{minipage}{1.0\linewidth}
    \makebox[.43\linewidth]{\includegraphics[width=.38\linewidth]{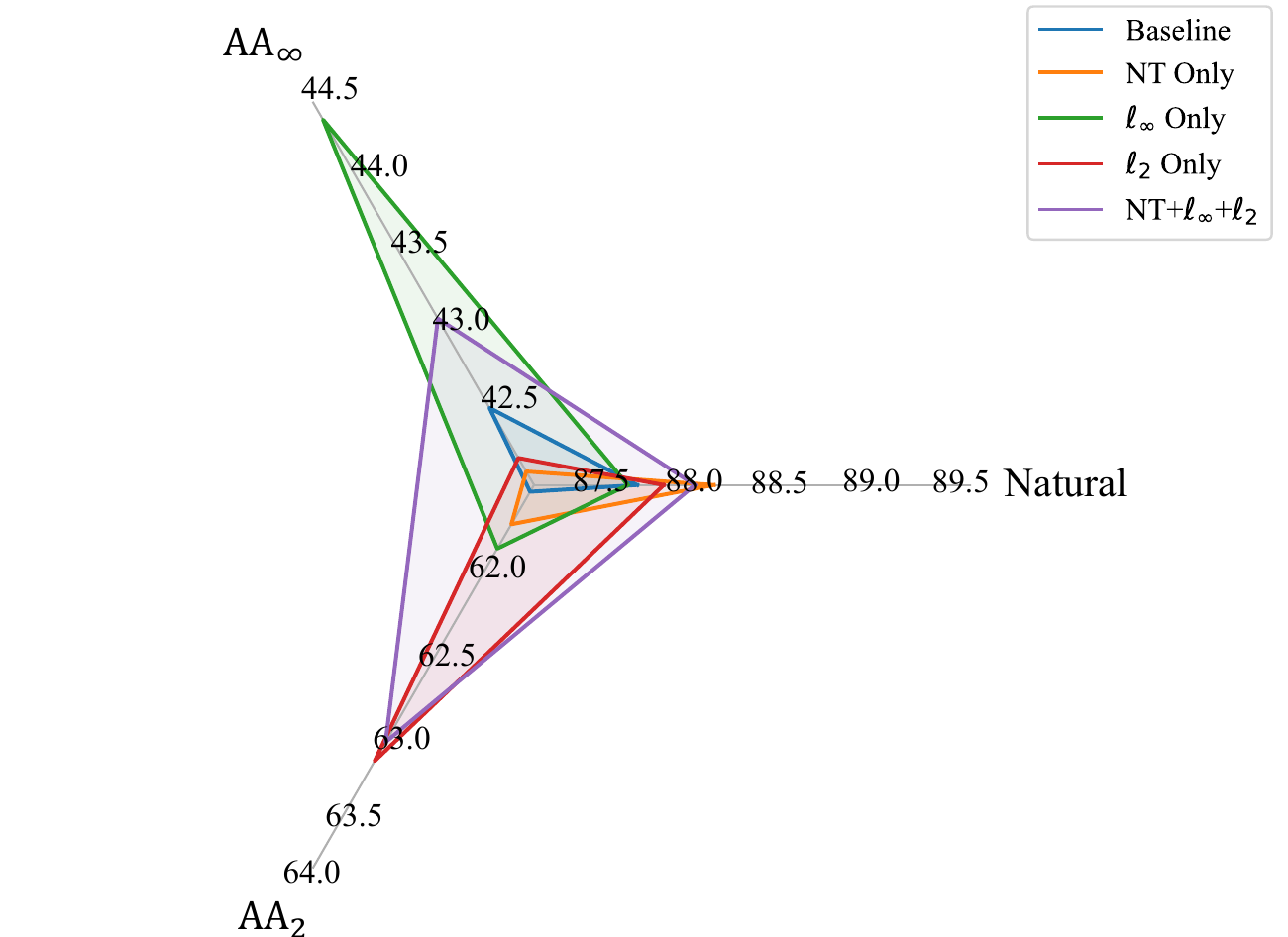}}%
    \hfill
    \makebox[.53\linewidth]{\includegraphics[width=.46\linewidth]{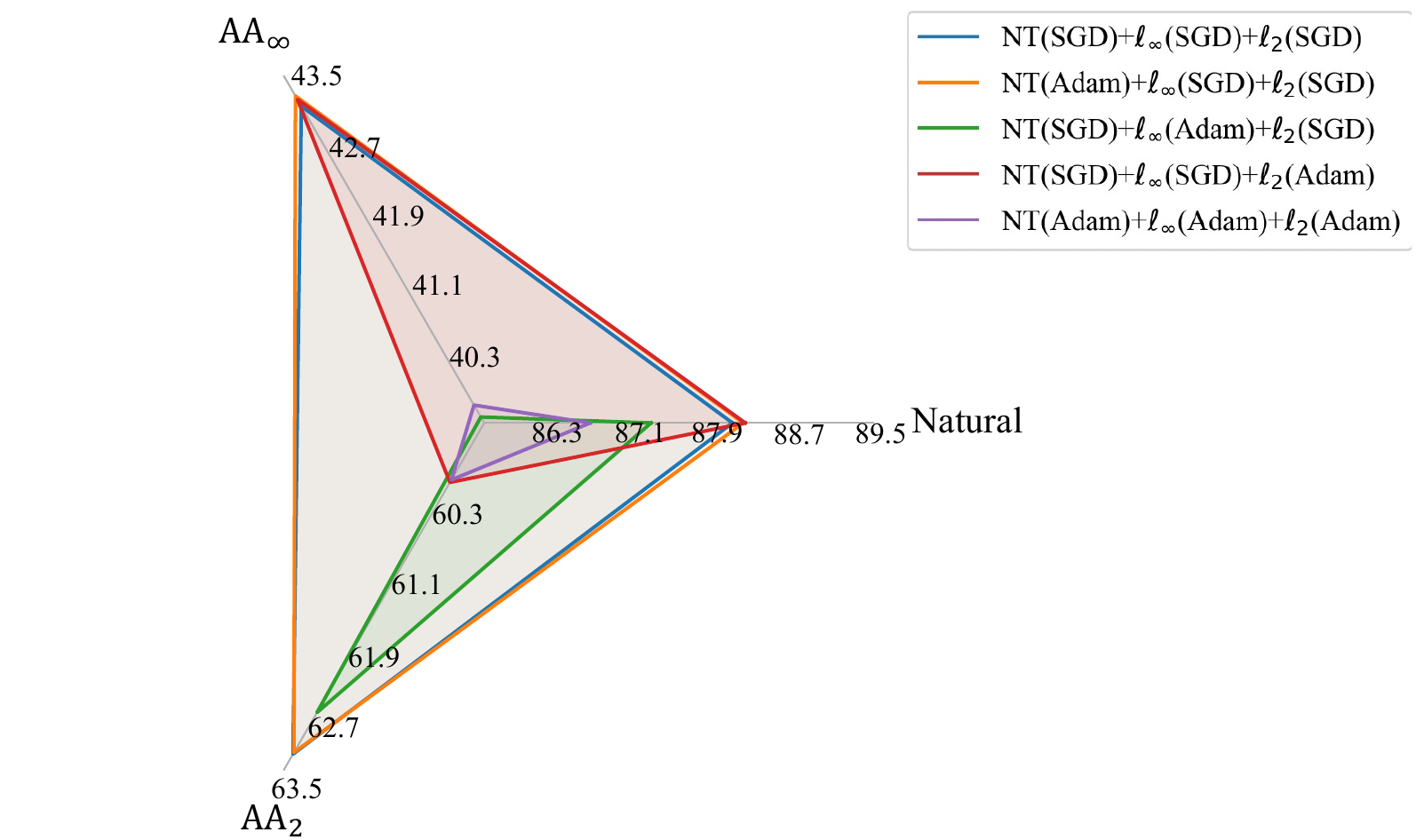}}%
    
    \makebox[.38\linewidth]{\small (a) Weight Averaging}\hspace{20pt}
    \makebox[.45\linewidth]{\small (b) Optimizer}%
  \end{minipage}%
\caption{Performance of Generalist-T under different configurations of (a) weight averaging and (b) optimizers. 
}\label{fig:optim_triple}
\vspace{-10pt}
\end{figure}

\vspace{-5pt}
\section{Customized Policies for Individuals} 
\label{sec:customized}

As discussed above, one of the key advantages of Generalist over standard joint training frameworks is its flexibility: each base learner can adopt a customized optimization strategy tailored to its specific task, rather than sharing a uniform strategy across all tasks.
In this section, we investigate whether such task-specific customization further enhances performance when Generalist is combined with diverse training techniques. For brevity, we focus on Generalist-T as a representative case, and refer to Appendix \ref{app:double_policy} for results on Generalist-D.

\vspace{-5pt}
\subsection{Weight Averaging}
Recent studies have demonstrated that weight averaging (WA) can substantially enhance both natural and robust generalization~\cite{DBLP:journals/corr/abs-2103-01946,DBLP:conf/uai/IzmailovPGVW18,wang2022selfensemble}. WA aggregates model parameters across training checkpoints to form an implicit ensemble, thereby stabilizing optimization and improving convergence. However, in traditional joint training frameworks, WA often fails to simultaneously benefit both accuracy and robustness. 

Therefore, we apply WA independently to each base learner in Generalist. The results for Generalist-T are presented in Figure~\ref{fig:optim_triple}(a). We evaluate several configurations: applying WA to a single base learner (NT Only, $\ell_\infty$ Only, or $\ell_2$ Only) and applying WA to all base learners simultaneously (NT+$\ell_\infty$+$\ell_2$). As illustrated in Figure~\ref{fig:optim_triple}(a), Generalist-T equipped with WA across all base learners achieves the most balanced and superior performance compared to the baseline. In contrast, applying WA to only one base learner leads to asymmetric improvements. When WA is applied solely to the NT learner, natural accuracy increases, but AA$_\infty$ declines. Applying WA only to the $\ell_\infty$ learner enhances AA$_\infty$ but reduces natural accuracy, whereas equipping only the $\ell_2$ learner raises AA$_2$ but simultaneously decreases AA$_\infty$. These results indicate that partial use of WA disrupts synchronization among base learners, as those equipped with WA converge faster on their subtasks, causing misalignment in learning dynamics. Conversely, enabling WA for all base learners ensures coordinated optimization and leads to a more stable and well-generalized global model.

\subsection{Different Optimizers}

We further examine the impact of using different optimizers for individual base learners. Specifically, we consider SGD with momentum and Adam under a piecewise learning rate schedule as the baselines. The initial learning rate for Adam is set to 0.0001. We then alternately substitute the optimizer of each base learner while keeping the others unchanged. The results are shown in Figure~\ref{fig:triple_gamma2}. 

As observed, applying Adam to the NT learner while keeping SGD for the $\ell_\infty$ and $\ell_2$ learners yields the most balanced overall performance, achieving the highest natural accuracy with comparable robustness. In contrast, assigning Adam to the adversarial learners (either $\ell_\infty$ or $\ell_2$) noticeably degrades robustness. For instance, when Adam is used for the $\ell_\infty$ learner, AA$_\infty$ drops significantly, and applying it to the $\ell_2$ learner similarly weakens AA$_2$ relative to the all-SGD configuration. These results suggest that Adam benefits natural accuracy but is less suited for adversarial training, where SGD provides more stable and consistent updates. Overall, this experiment confirms the advantage of Generalist’s decoupled optimization scheme: each base learner can adopt an optimizer best aligned with its task characteristics. By assigning Adam to natural learning and SGD to adversarial learners, Generalist effectively leverages the strengths of both optimizers to achieve a better trade-off between accuracy and robustness.

\section{Interpretable Analysis}
\label{sec:visual}

\begin{figure}[!t]
\centering
  \begin{minipage}{\linewidth}
    \makebox[.5\linewidth]{\includegraphics[width=.5\linewidth]{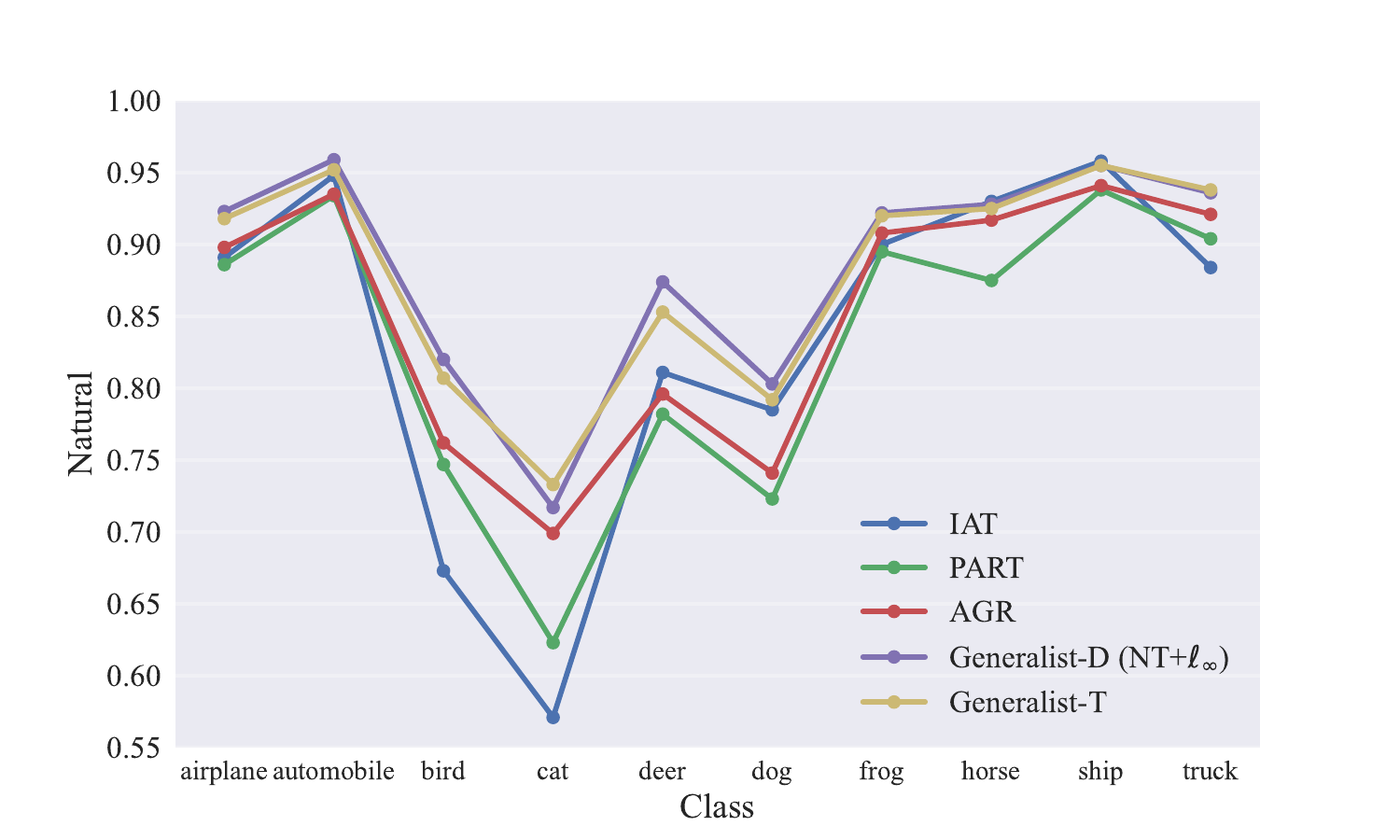}}%
    \hfill
    \makebox[.5\linewidth]{\includegraphics[width=.5\linewidth]{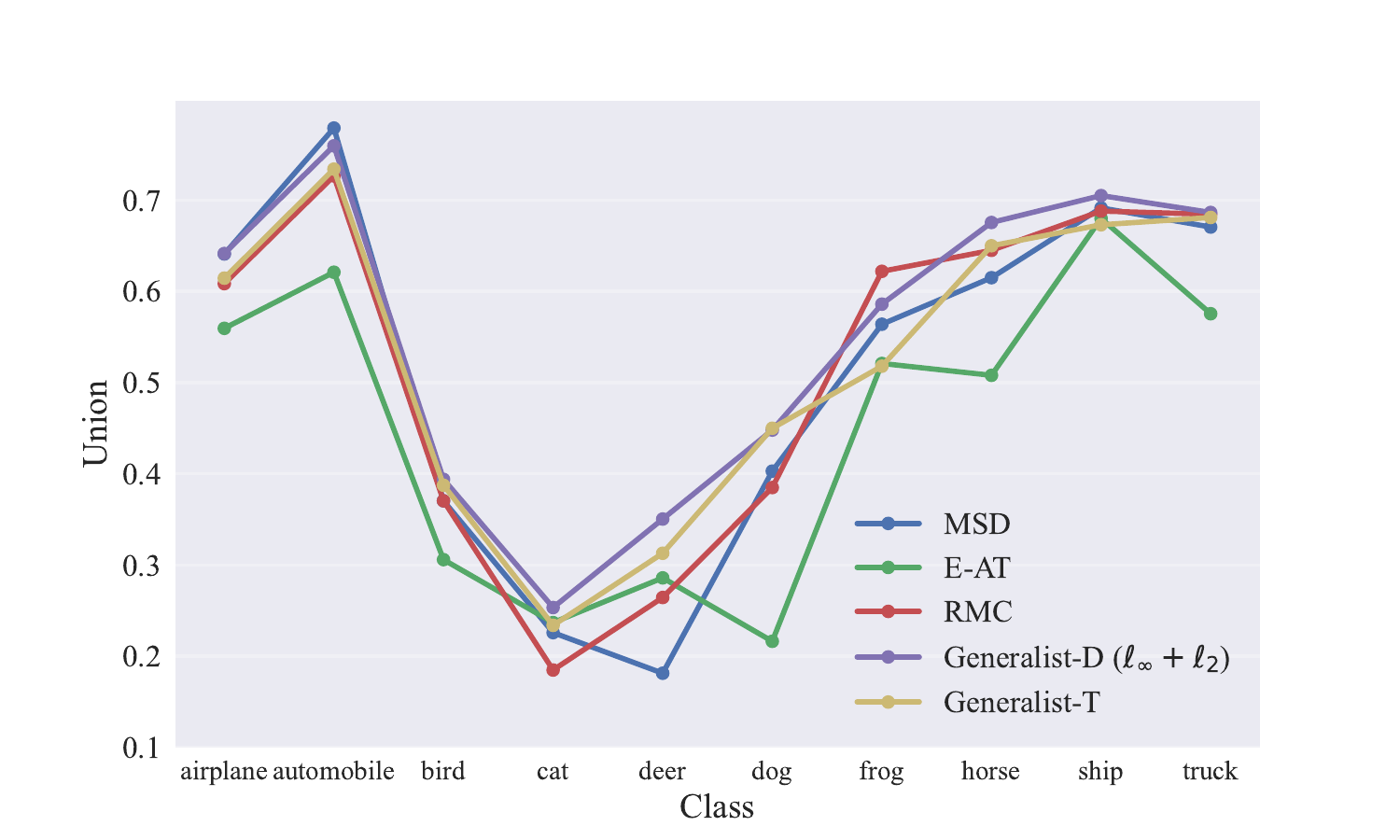}} 
    \makebox[.5\linewidth]{\small (a)}%
    \makebox[.5\linewidth]{\small (b)}%
  \end{minipage}%
\caption{Analysis of class-wise predictions that different robust classifiers have on CIFAR-10 using ResNet-18. (a) Class-wise natural accuracy of AT variants for the accuracy improvement goal. (b) Class-wise union robustness of AT variants for the multi-norm robustness goal.}\label{fig:cls}
\vspace{-10pt}
\end{figure}

While the previous sections demonstrate the superior quantitative performance of Generalist, it remains essential to understand how such improvements arise. To this end, we conduct an interpretable analysis to examine the representations learned by Generalist from both quantitative and qualitative perspectives.

We first investigate the class-wise behavior of different adversarial training methods to reveal which categories benefit most from Generalist’s design. We then complement this with a visual interpretability study using Grad-CAM, comparing the attention maps of Generalist and baseline models on both natural and adversarial examples.

\subsection{Class-wise Behavior Analysis}

Considering that Generalist achieves remarkable performance in mitigating both the natural–robustness and multi-norm trade-offs, it is instructive to analyze in detail how these gains are distributed across different categories. To this end, we examine class-wise prediction behaviors of robust classifiers on CIFAR-10. We conduct experiments with six representative baselines and three variants of Generalist, divided into two groups according to their learning objectives. The first group—including IAT, PART, AGR, Generalist-D ($NT+\ell_\infty$), and Generalist-T—focuses on improving natural accuracy while maintaining competitive $\ell_\infty$ robustness. The second group—including MSD, E-AT, RMC, Generalist-D ($\ell_\infty+\ell_2$), and Generalist-T—targets robustness generalization across multiple perturbation norms. Their class-wise performances are visualized in Figure~\ref{fig:cls}.

From Figure~\ref{fig:cls}(a), we observe that all models exhibit noticeable drops in accuracy for bird, cat, deer, and dog—the so-called hard classes identified in prior work~\cite{wang2019symmetric}. Nevertheless, Generalist-D ($NT+\ell_\infty$) and Generalist-T consistently achieve higher natural accuracy across almost all categories, and the gains are particularly significant on these hard classes, while maintaining comparable performance on the easier ones such as automobile, ship, and truck. A similar trend is observed in Figure~\ref{fig:cls}(b): The benefits of  Generalist-D ($\ell_\infty+\ell_2$) and Generalist-T are prominent in the hard classes (bird, cat, deer, dog). These results suggest that Generalist not only improves overall robustness but also alleviates class-specific vulnerability, leading to more balanced and consistent generalization across categories.

\subsection{Visual Interpretability Analysis}

To better understand these behavioral differences, we visualize representative samples from the CIFAR-10 test set that are misclassified by baseline methods but correctly predicted by Generalist-D and Generalist-T, including both natural examples (Figure~\ref{fig:heat}(a)) and $\ell_2$-bounded adversarial examples crafted by PGD$^{20}_2$ (Figure~\ref{fig:heat}(b)). Using Grad-CAM~\cite{selvaraju2017grad}, we examine the spatial attention regions of each model to understand where they focus when making predictions.

Baseline AT methods, though robust to certain perturbations, often rely on spurious background correlations. For example, in Figure~\ref{fig:heat}(a), PART misclassifies an airplane as a bird because it attends to the blue-sky background, while FAT and AGR also overemphasize irrelevant contextual textures. In Figure~\ref{fig:heat}(b), E-AT and RMC fail on dog examples where the background or color cues overlap, showing that their robustness is largely context-dependent. In contrast, Generalist-D and Generalist-T consistently focus on the foreground object regions, capturing structural and shape cues that remain stable across both natural and adversarial domains. 

Overall, these qualitative observations reinforce the quantitative findings: Generalist effectively filters out background noise and learns foreground-centered, semantically meaningful representations that generalize across diverse perturbations.

\begin{figure*}[!t]
\centering
  \begin{minipage}{1.0\linewidth}
      \hspace{30pt}
    \makebox[.4\linewidth]{\includegraphics[width=.4\linewidth]{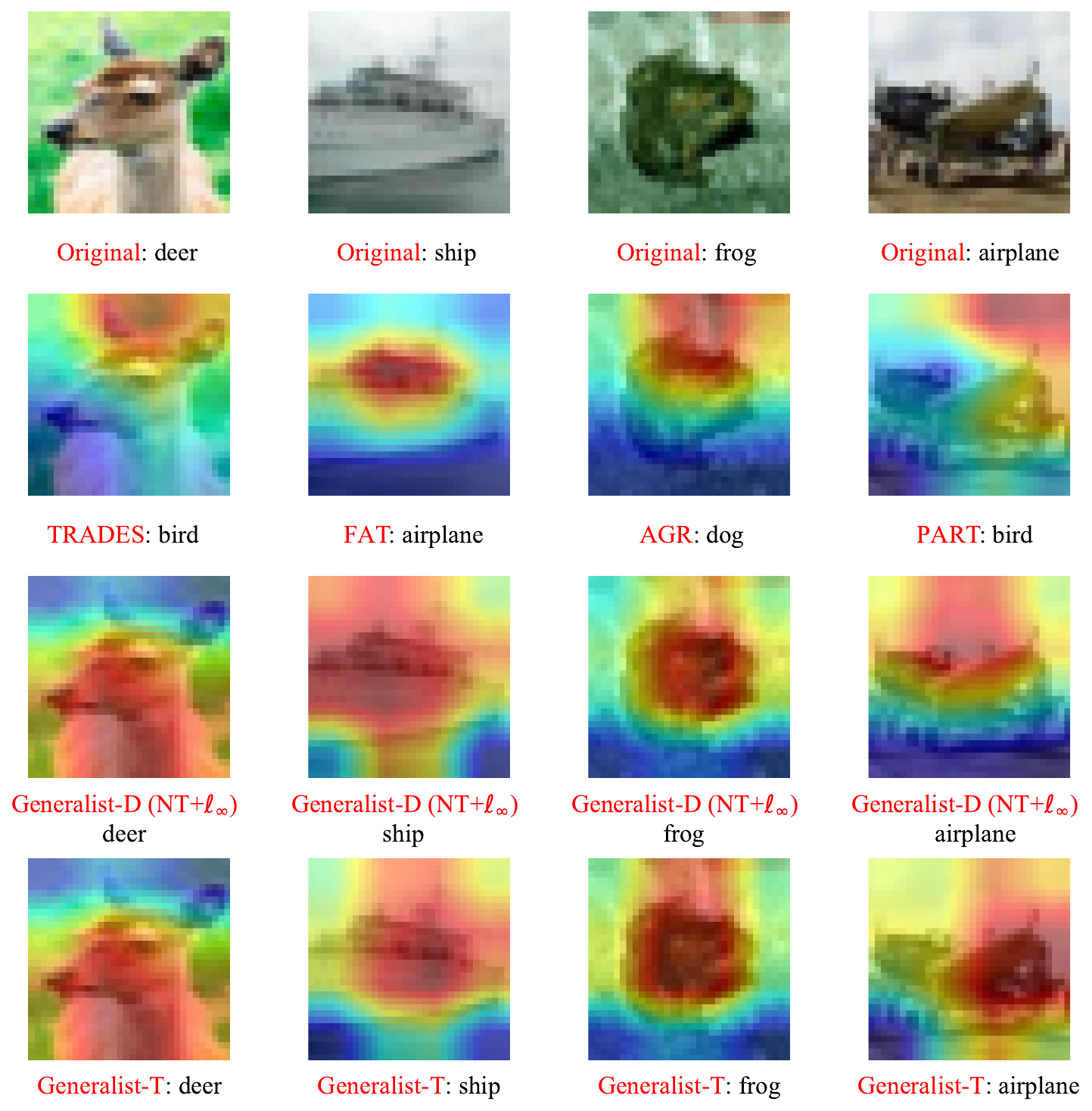}}%
    \hspace{50pt}
    \makebox[.4\linewidth]{\includegraphics[width=.4\linewidth]{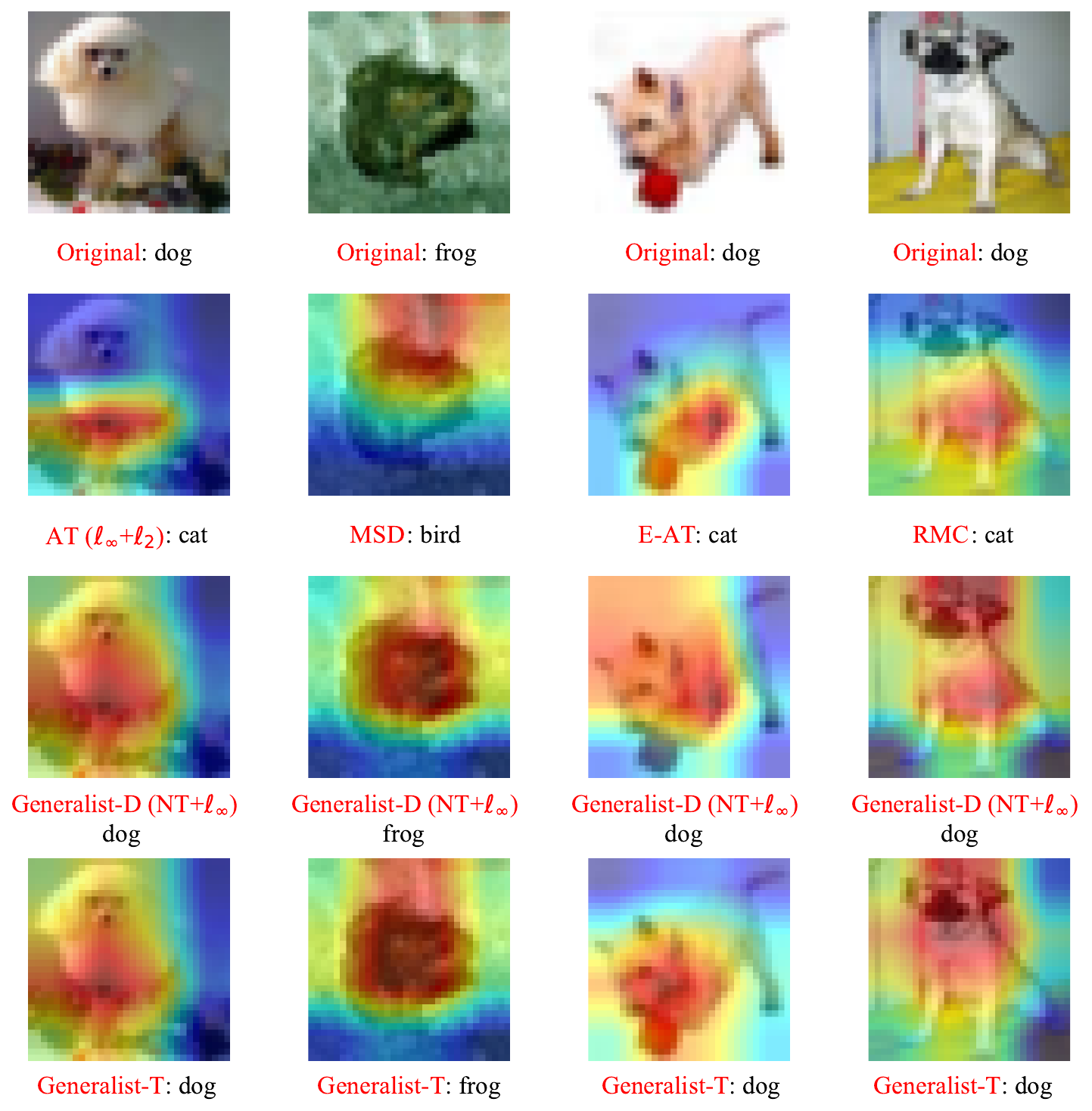}} 
    
    \hspace{5pt}
    \makebox[.5\linewidth]{ (a)}%
    \makebox[.5\linewidth]{ (b)}%
  \end{minipage}%
    \caption{Heatmap visualizations on the CIFAR-10 test set using Grad-CAM. (a) Natural samples misclassified by baselines but correctly recognized by Generalist-D ($NT+\ell_\infty$) and Generalist-T. (b) Adversarial samples carfted by PGD$_{2}^{20}$ misclassified by baselines while Generalist-D ($\ell_\infty+\ell_2$) and Generalist-T make correct predictions.}\label{fig:heat}
    \vspace{-10pt}
\end{figure*}

\section{Conclusion}
\label{conclusion}
In this paper, we propose a multi-expert framework named Generalist to alleviate both the natural-robustness and multi-norm tradeoff issues, which trains multiple base learners responsible for complementary fields and collects their parameters to construct a global learner. By decoupling from the joint training paradigm, each base learner can wield customized strategies based on data distribution. According to its detailed applicable scenarios, we develop three variants from one framework including: Generalist-D ($NT+\ell_\infty$), Generalist-D ($\ell_\infty+\ell_2$) and Generalist-T ($NT+\ell_\infty+\ell_2$). We provide not only theoretical analysis to justify the effectiveness of task-aware strategies but also extensive experiments to show the extraordinary performances of Generalist on both small and big datasets. In addition, the extensive experiments on the OOD datasets reveal that the knowledge learned by Generalist can be generalized to resisting attacks from unseen domains. Our further ablation studies also show the advantage of Generalist in assigning customized policies for individual learners and capturing the invariant robust features. We hope Generalist will serve as a foundation for the development of fully robust classifiers in the future.

\bibliography{ref}
\bibliographystyle{IEEEtran}

\vspace{-32pt}
\begin{IEEEbiography}
[{\includegraphics[width=1in,height=1.25in,clip,keepaspectratio]{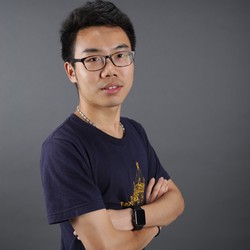}}]{Yisen Wang} received the Ph.D. degree from Tsinghua University in 2018. He is currently an Assistant Professor at Peking University. His research interest includes machine learning and deep learning, such as adversarial learning, graph learning, and weakly/self-supervised learning.
\end{IEEEbiography}
\vspace{-32pt}
\begin{IEEEbiography}[{\includegraphics[width=1in,height=1.25in,clip,keepaspectratio]{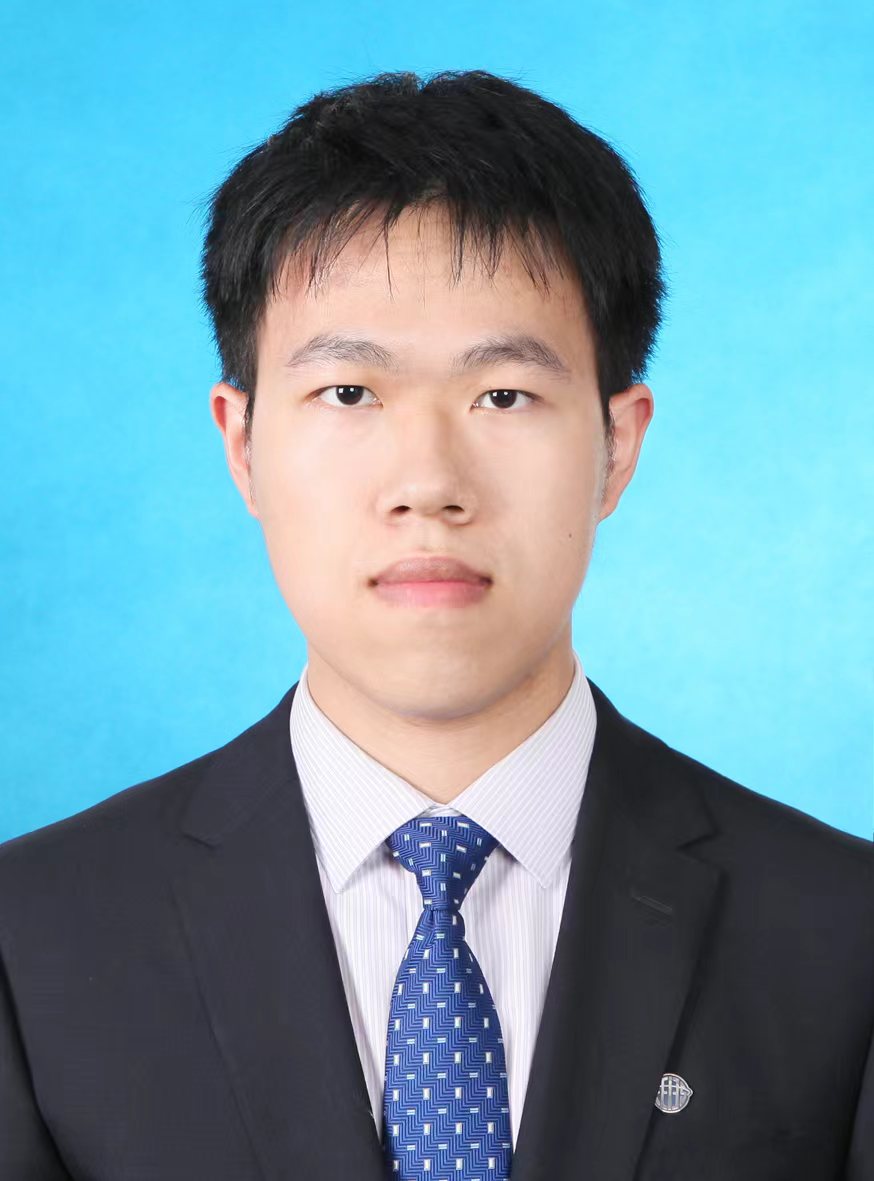}}]{Yichuan Mo} received the B.E. degree from Shanghai Jiao Tong University in 2022. He is currently a Ph.D. candidate at Peking University. His research interests include adversarial learning, model robustness and trustworthy AI. 
\end{IEEEbiography}
\vspace{-32pt}

\begin{IEEEbiography}[{\includegraphics[width=1in,height=1.25in,clip,keepaspectratio]{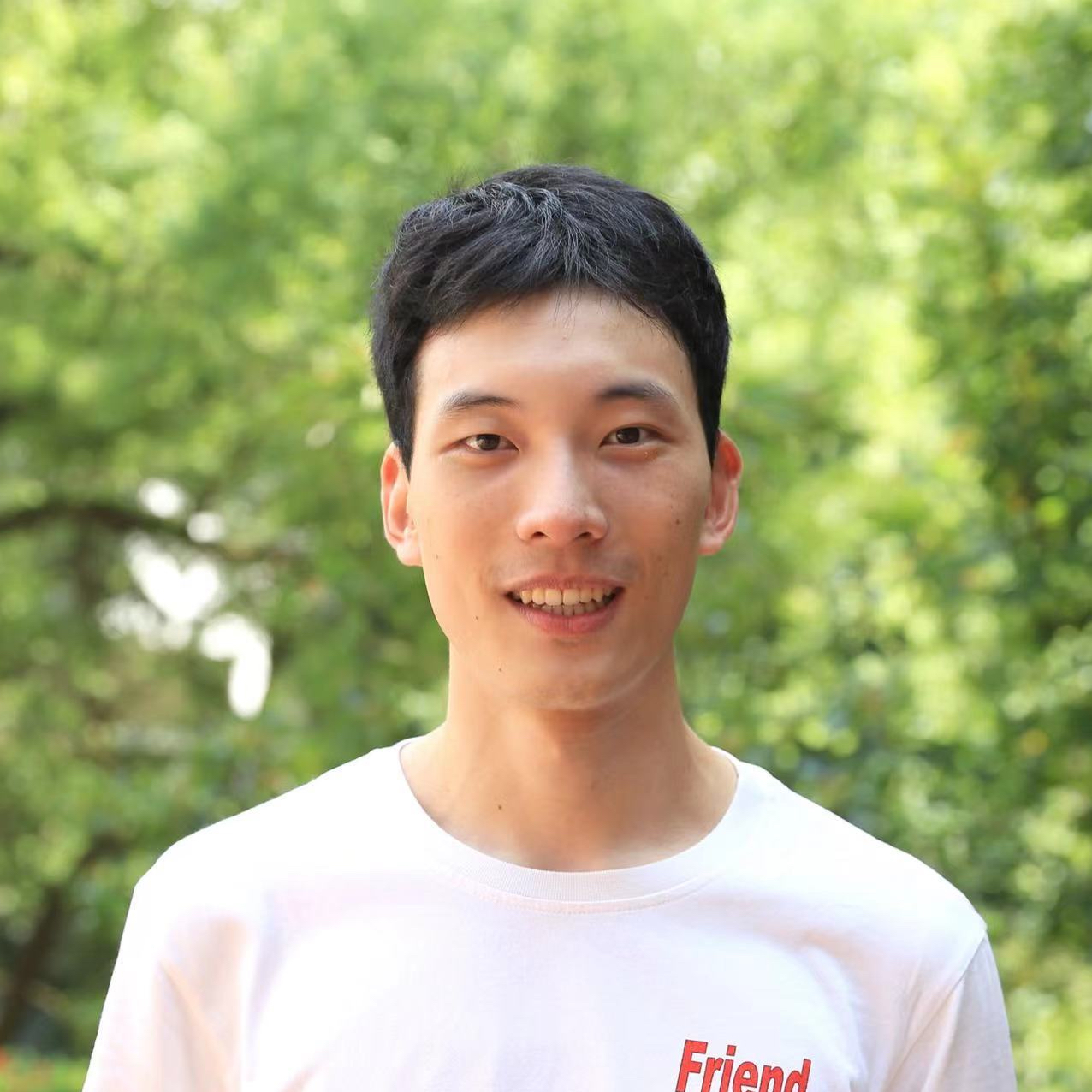}}]{Hongjun Wang} received the B.E. and MPhil degrees from Sun Yat-sen University in 2018 and 2021. He is currently a Ph.D. candidate at The University of Hong Kong. His research interests include open-world scene understanding and distribution shifts. 
\end{IEEEbiography}
\vspace{-32pt}

\begin{IEEEbiography}[{\includegraphics[width=1in,height=1.25in,clip,keepaspectratio]{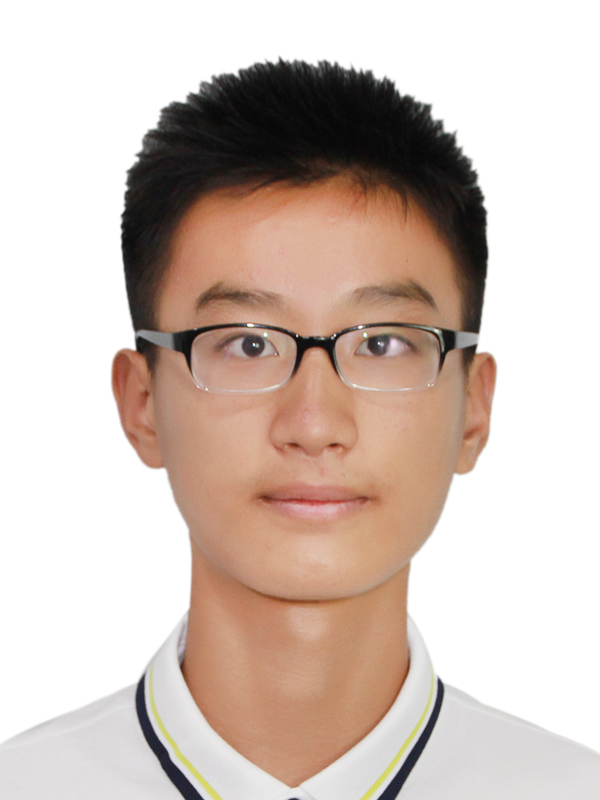}}]{Junyi Li} is an undergraduate student majoring in Mathematics and Applied Mathematics at Peking University. His research interests include trustworthy AI and machine learning.
\end{IEEEbiography}
\vspace{-32pt}

\begin{IEEEbiography}[{\includegraphics[width=1in,height=1.25in,clip,keepaspectratio]{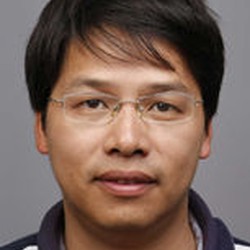}}]{Zhouchen Lin} (M'00-SM'08-F'18) received the Ph.D. degree from Peking University in 2000. He is currently a professor at Peking University. His research interests include computer vision, image processing, machine learning, pattern recognition, and numerical optimization. He was an associate editor of the IEEE Transactions on Pattern Analysis and Machine Intelligence and currently is an associate editor of the International Journal of Computer Vision. He is a Fellow of IAPR and IEEE.
\end{IEEEbiography}

\newpage

\appendices

\section{Proofs of Theoretical Results}
\label{app:theory}
\subsection{Proof of Claim in Section \ref{sec:global}}
\label{apd:b1}
\label{apd:global}
\begin{proof}
At epoch $t$, the parameters of the global learner are distributed to the experts and each expert train from this initialization with $c$ steps by calculating the gradients. Following \cite{DBLP:journals/corr/abs-1803-02999}, we approximate the update performed by the initialization based on the Taylor expansion:
\begin{equation}
\label{eqn:proof_taylor}
\small
\begin{aligned}
g^{t+c}&=\ell^{\prime}\left(\boldsymbol\theta^{t+c}\right) \\&=\ell^{\prime}\left(\boldsymbol\theta^{t}\right)+\ell^{\prime \prime}\left(\boldsymbol\theta^{t}\right)\left(\boldsymbol\theta^{t+c}-\boldsymbol\theta^{t}\right)+O\left(\left\|\boldsymbol\theta^{t+c}-\boldsymbol\theta^{t}\right\|^{2}\right)\\
&\left.=\bar{g}^{t}+\bar{H}^{t}\left(\boldsymbol\theta^{t+c}-\boldsymbol\theta^{t}\right)+O\left(\tau^{2}\right)\right) \\
&=\bar{g}^{t}-\tau \bar{H}^{t} \sum_{j=t}^{t+c} g^{j}+O\left(\tau^{2}\right) \\
&=\bar{g}^{t}-\tau \bar{H}^{t} \sum_{j=t}^{t+c} \bar{g}^{j}+O\left(\tau^{2}\right).
\end{aligned}
\end{equation}
It should be noted that $g$ and $\bar{g}$ in the above equation correspond to gradient at $\theta^{t+c}$ and $\theta^{t}$, respectively. We will continue to use these notation in the following proof.
Recalling that $\mathcal{Z}^{i}$ represents an optimizer that updates the parameter vector at the $t$-th step: $\mathcal{Z}^{i}(\boldsymbol\theta,\tau)=\boldsymbol\theta-\tau\ell^{\prime}(\boldsymbol\theta)$. For each base-learner, we approximate the gradient at intervals: 
\begin{equation}
\begin{aligned}
g_{val}&=\frac{\partial}{\partial \boldsymbol\theta^{t}} \ell\left(\boldsymbol\theta^{t+c}\right) \\&=\frac{\partial}{\partial \boldsymbol\theta^{t}} \ell\left(\mathcal{Z}^{t+c-1}\left(\mathcal{Z}^{t+c-2}\left(\ldots\left(\mathcal{Z}^{t}\left(\boldsymbol\theta^{t}\right)\right)\right)\right)\right) \\
&={\mathcal{Z}^{\prime}}^{t}\left(\boldsymbol\theta^{t}\right) \cdots {\mathcal{Z}^{\prime}}^{t+c-1}\left(\boldsymbol\theta^{t+c-1}\right) \ell^{\prime}\left(\boldsymbol\theta^{t+c}\right) \\
&=\left(I-\tau \ell^{\prime \prime}\left(\boldsymbol\theta^{t}\right)\right) \cdots\left(I-\tau \ell^{\prime \prime}\left(\boldsymbol\theta^{t+c-1}\right)\right) \ell^{\prime}\left(\boldsymbol\theta^{t+c}\right) \\
&=\left(\prod_{j=t}^{t+c-1}\left(I-\tau \ell^{\prime \prime}\left(\boldsymbol\theta^{j}\right)\right)\right) g^{t+c}.
\end{aligned}
\end{equation}

Here $g_{\mathrm{val}}$ denotes the validation gradient, i.e., the gradient obtained after initializing the base learner with the global parameter $\theta_g$ and further training it for $c$ steps, which characterizes how the global initialization influences subsequent task-specific updates.

Replacing $\ell^{\prime \prime}\left(\boldsymbol\theta^{j}\right)$ with $\bar{H}^{j}$ and substituting $g^{t+c}$ for Eq. \ref{eqn:proof_taylor}, we expand to leading order:
\begin{equation}
\small
\begin{aligned}
g_{val}&=\left(\prod_{j=t}^{t+c-1}\left(I-\tau \bar{H}^{j}\right)\right)\left(\bar{g}^{t+c}-\tau \bar{H}^{t+c} \sum_{j=t}^{t+c-1} \bar{g}^{j}\right)+O\left(\tau^{2}\right) \\
&=\left(I-\tau \sum_{j=t}^{t+c-1} \bar{H}^{j}\right)\left(\bar{g}^{t+c}-\tau \bar{H}^{t+c} \sum_{j=t}^{t+c-1} \bar{g}^{j}\right)+O\left(\tau^{2}\right) \\
&=\bar{g}^{t+c}-\tau \sum_{j=t}^{t+c-1} \bar{H}^{j} \bar{g}^{t+c}-\tau \bar{H}^{t+c} \sum_{j=t}^{t+c-1} \bar{g}^{j}+O\left(\tau^{2}\right)
\end{aligned}
\end{equation}
Therefore, we take the expectation of $g_{val}$ over steps, and obtain:
\begin{equation}
\small
\begin{aligned}
\mathbb{E}\left[g_{val}\right]&=\mathbb{E}\left[\bar{g}^{t+c}\right]-\tau\mathbb{E}\left[\sum_{j=t}^{t+c-1} \bar{H}^{j} \bar{g}^{t+c} \right.\\ &+\left.\bar{H}^{t+c} \sum_{j=t}^{t+c-1}\bar{g}^{j}\right]+\mathbb{E}\left[O\left(\tau^{2}\right)\right]
\end{aligned}
\end{equation}
Recalling that $\boldsymbol\theta_{g}$ is mixed by $\boldsymbol\theta_1$, $\boldsymbol\theta_2$,$\cdots$, $\boldsymbol\theta_{|\mathcal{A}|}$. For simplicity of exposition, we use $\gamma_1$, $\gamma_2$, $\cdots$,$\gamma_{|\mathcal{A}|}$ to stand for the scalar factors, meaning $\boldsymbol\theta_{g}=\sum\limits_{\mathcal{W}=1}^{|\mathcal{A}|}\gamma_\mathcal{W}\boldsymbol\theta_\mathcal{W}$. Ignoring the higher order terms, for each expert initialized by the global learner (e.g. $\boldsymbol\theta_{n}$), we have:
\begin{equation}
\begin{aligned}
\boldsymbol\theta_{n}&=\boldsymbol\theta_{g}-\mathbb{E}_n\left[g_{val}\right]\\&=\sum\limits_{\mathcal{W}=1}^{|\mathcal{A}|}\gamma_\mathcal{W}\boldsymbol\theta_\mathcal{W}-[\mathbb{E}\left[\bar{g}^{t+c}_n\right]\\&-\tau_n\mathbb{E}\left[\sum_{j=t}^{t+c-1} \bar{H}^{j} \bar{g}^{t+c}_n+\bar{H}^{t+c} \sum_{j=t}^{t+c-1} \bar{g}^{j}_n\right]] \\
&=[\gamma_{n}\boldsymbol\theta_{n}-\mathbb{E}\left[\bar{g}^{t+c}_n\right]] + [\sum\limits_{\substack{\mathcal{W}=1 \\ \mathcal{W}\neq n}}^{|\mathcal{A}|}
\gamma_\mathcal{W}\boldsymbol\theta_\mathcal{W}\\&-\tau_n\mathbb{E}\left[\bar{H}^{t+c} \sum_{j=t}^{t+c-1} \bar{g}^{j}_n+\sum_{j=t}^{t+c-1} \bar{H}^{j} \bar{g}^{t+c}_n\right]] \\
&=[\gamma_{n}\boldsymbol\theta_{n}-\sum_{i=t}^{t+c-1} \bar{g}^{i}_{n}] +[\sum\limits_{\substack{\mathcal{W}=1 \\ \mathcal{W}\neq n}}^{|\mathcal{A}|}
\gamma_\mathcal{W}\boldsymbol\theta_\mathcal{W}\\&-\tau_n\left(2\bar{H}^{t} \sum_{j=t}^{t+c-1} \bar{g}^{j}_n-\bar{H}^{t}\sum_{i=t}^{t+c-1} \sum_{j=1}^{i-1} \bar{H}^{i} \bar{g}^{j}_n\right)] (\text{for} \ c\ge2).
\end{aligned}
\end{equation}
The first term pushes $\theta_{n}$ to move forward the minimum of its assigned loss over its data distribution; while the second term improves generalization by increasing the inner product between gradients of different mini-batches and updating the parameters from the other task.
\end{proof}
    
\subsection{Proof of Theorem \ref{theorem:Theorem-1}}
\label{apd:b2}
Before we present the proof of the Theorem we present useful intermediate results which we require in our proof.

\begin{proposition}
\label{pro:1}
Consider a sequence of loss functions ${\ell_a: \Theta\mapsto [0, 1]}_{a\in \mathcal{A}}$ drawn i.i.d. from some distribution $\mathcal{L}$ is given to an algorithm that generates a sequence of hypotheses $\left\{\boldsymbol\theta_{a} \in \Theta\right\}_{a \in\mathcal{A}}$ then the following inequality each hold w.p. $1-\delta$:
\begin{equation}
\frac{1}{T} \sum_{t=1}^{T} \underset{\ell \sim D}{\mathbb{E}} \ell\left(\boldsymbol\theta^{t}\right) \leq \frac{1}{T} \sum_{t=1}^{T} \ell^{t}\left(\boldsymbol\theta^{t}\right)+\sqrt{\frac{2}{T}\log \frac{1}{\delta}}.
\end{equation}
\end{proposition}
\begin{proof}
The proof of the Proposition can be directly derived from the Proposition 1 in \cite{DBLP:journals/tit/Cesa-BianchiCG04}.
\end{proof}
Then we could immediately obtain the below inequality by the symmetric version of the Azuma-Hoeffding inequality \cite{Azuma1967WEIGHTEDSO}
\begin{remark}
\label{remark1}
\begin{equation}
\frac{1}{T} \sum_{t=1}^{T} \underset{\ell \sim \mathcal{L}}{\mathbb{E}} \ell\left(\boldsymbol\theta^{t}\right) \geq \frac{1}{T} \sum_{t=1}^{T} \ell^{t}\left(\boldsymbol\theta^{t}\right)-\sqrt{\frac{2}{T}\log \frac{1}{\delta}}.
\end{equation}
\end{remark}
In short, the proposition and remark above jointly indicate
the following centralized random variable has a Sub-Guassian tail.
\begin{equation}
 \sum_{t=1}^{T} \ell^{t}\left(\boldsymbol\theta^{t}\right)- \sum_{t=1}^{T} \underset{\ell \sim \mathcal{L}}{\mathbb{E}} \ell\left(\boldsymbol\theta^{t}\right)
\end{equation}

Finally, we give the definition of the regret of minimizing any subproblem:
\label{def:1}
\begin{definition}
(\textbf{Subproblem Regret}) Consider an algorithm generates the trajectory of states $\left\{\boldsymbol\theta^{t} \in \Theta\right\}_{t \in[T]}$, the regret of such an algorithm on loss function $\left\{\ell^{t}\right\}_{t \in[T]}$ is:
\begin{equation}
\bar{\mathbf{R}}=\sum_{t=1}^{T} \ell^{t}\left(\boldsymbol\theta^{t}\right)-\inf _{\boldsymbol\theta^{\star} \in \Theta} \sum_{t=1}^{T} \ell^{t}(\boldsymbol\theta).
\vspace{-5pt}
\end{equation}
\end{definition}
\label{thm:1}
\begin{theorem}
(Restated) Consider an algorithm with regret bound $R_{T}$ that generates the trajectory of states for $|\mathcal{A}|$ base learners, for any parameter state $\boldsymbol\theta \in \Theta$, given a sequence of convex surrogate evaluation functions ${\ell: \Theta\mapsto [0, 1]_{a\in \mathcal{A}}}$ drawn i.i.d. from some distribution $\mathcal{L}$, the expected error of the global learner $\boldsymbol\theta_{g}$ on both tasks over the test set can be bounded with probability at least $1-\delta$:
\begin{equation}
\underset{\ell \sim \mathcal{L}}{\mathbb{E}} \ell\left(\boldsymbol\theta_{g}\right) \leq \underset{\ell \sim \mathcal{L}}{\mathbb{E}} \ell\left(\boldsymbol\theta\right)+\frac{\mathbf{R}_{T}}{T}+2\sqrt{\frac{2}{T}\log \frac{1}{\delta}}.
\end{equation}
\end{theorem}
\begin{proof}
We denote ${\theta}^t$ through $t=1,\cdots,T$ as the update trajectory of ${\theta}_g$. The outline of the proof is as follows. We first construct an upper bound for $\frac{1}{T} \sum_{t=1}^{T} \underset{\ell \sim \mathcal{L}}{\mathbb{E}}\ell(\boldsymbol{\theta}^t)$ using $\bar{R}$ and then switch $\bar{R}$ to $ R_{T}$. After that, we Establish a connection between $\underset{\ell \sim \mathcal{L}}{\mathbb{E}} \ell\left(\boldsymbol\theta_{g}\right)$ and above results using Jensen's inequality.
From Proposition \ref{pro:1} and Remark \ref{remark1}, the following inequality holds with possibility at least $1-\delta$ for any parameter state $\boldsymbol\theta\in\Theta$:

\begin{equation}
\begin{aligned}
& \frac{1}{T} \sum_{t=1}^{T} \underset{\ell \sim \mathcal{L}}{\mathbb{E}}  \ell(\boldsymbol{\theta}^t) \leq \frac{1}{T} \sum_{t=1}^{T} \ell^t(\theta^t) + \sqrt{\frac{2}{T} \log \frac{1}{\delta}} \\
&=\frac{1}{T} \sum_{t=1}^{T} \ell^t(\theta) + (\frac{1}{T} \sum_{t=1}^{T} \ell^t(\theta^t) - \frac{1}{T} \sum_{t=1}^{T} \ell^t(\theta)) \\ &+\sqrt{\frac{2}{T} \log \frac{1}{\delta}}\\
&\leq \frac{1}{T} \sum_{t=1}^{T} \ell^t(\theta) + \frac{\bar{\mathbf{R}}}{T} + \sqrt{\frac{2}{T} \log \frac{1}{\delta}} \\
&\leq \underset{\ell \sim \mathcal{L}}{\mathbb{E}}  \ell(\boldsymbol{\theta}) + \frac{\bar{\mathbf{R}}}{T} + 2\sqrt{\frac{2}{T} \log \frac{1}{\delta}}.
\end{aligned}
\label{eq:eq1}
\vspace{-3pt}
\end{equation}
Noticed that $ R_{T}$ describes the performance gap between the updating trajectory and theoretically optimal parameters for each task. It turns out that a large term will appear every c steps in $ R_{T}$, due to the frequency of communication in the algorithm is c. 
So it is obvious that:
\begin{equation}
\bar{\mathbf{R}} \leq \mathbf{R}_{T}
\end{equation}
We can derive the following inequality directly from Equation \ref{eq:eq1}:
\begin{equation}
\begin{aligned}
\frac{1}{T} \sum_{t=1}^{T} \underset{\ell \sim \mathcal{L}}{\mathbb{E}} \ell\left(\boldsymbol\theta^{t}\right) \leq & \underset{\ell \sim \mathcal{L}}{\mathbb{E}} \ell\left(\boldsymbol\theta\right)+\frac{\mathbf{R}_{T}}{T}+2\sqrt{\frac{2}{T}\log \frac{1}{\delta}}.
\end{aligned}
\label{eq:eq2}
\vspace{-10pt}
\end{equation}
Since we can treat ${\theta}^1,{\theta}^2,\cdots,{\theta}^T$ as a sequence that converges to $ {\theta}_{g} $, the average value of this sequence with length T is close to ${\theta}_{g}$. This is ensured by the well-known conclusion below:
\begin{equation}
\begin{aligned}
    \lim_{t \to \infty} \theta_t = \theta \quad \Rightarrow \quad \lim_{T \to \infty} \frac{1}{T} \sum_{t=1}^{T} \theta_t = \theta.
\end{aligned}
\end{equation}
Then, the above inequality Equation \ref{eq:eq2} can be further derived by the Jensen's inequality (convex surrogate functions could be selected to evaluate the test errors instead of the 0-1 loss):
\begin{equation}
\begin{aligned}
\underset{\ell \sim \mathcal{L}}{\mathbb{E}} \ell\left(\boldsymbol\theta_{g}\right)&\approx
\underset{\ell \sim \mathcal{L}}{\mathbb{E}} \ell\left(\frac{1}{T}
\sum_{t=1}^{T}\boldsymbol\theta^{t}\right) \leq \frac{1}{T} \sum_{t=1}^{T} \underset{\ell \sim \mathcal{L}}{\mathbb{E}} \ell\left(\boldsymbol\theta^{t}\right) \\&\leq \frac{1}{T} \sum_{t=1}^{T} \ell^{t}\left(\boldsymbol\theta\right)+\frac{\mathbf{R}_{T}}{T}+\sqrt{\frac{2}{T}\log \frac{1}{\delta}} \\
&\leq \underset{\ell \sim \mathcal{L}}{\mathbb{E}} \ell\left(\boldsymbol\theta\right)+\frac{\mathbf{R}_{T}}{T}+2\sqrt{\frac{2}{T}\log \frac{1}{\delta}}.
\end{aligned}
\end{equation}
Note that this inequality also holds when applying weight averaging technique to the base-learner, because weight averaging is the linear combination of all history states.
\end{proof}

\subsection{ Proof of Theorem \ref{thm:2}}
\label{pf:proof-global-stability}

\textbf{Setup and notation:}
Let the multi-task training collections be
\(\mathcal{D}=(\mathcal{D}_a)_{a=1}^{|\mathcal{A}|}\) and
\(\mathcal{D}'=(\mathcal{D}_a^{\prime})_{a=1}^{|\mathcal{A}|}\),
differing in exactly one example (in some task \(a^\star\)).
Denote by \(\theta_a=\theta(\mathcal{D}_a)\) and \(\theta_a'=\theta(\mathcal{D}_a^{\prime})\) the base parameters, and
\[
\theta_g=\sum_{a=1}^{|\mathcal{A}|}\gamma_a\,\theta_a,
\qquad
\theta_g'=\sum_{a=1}^{|\mathcal{A}|}\gamma_a\,\theta_a'.
\]
Let \(\bar{\theta}\) be the previous global iterate.
We write \(f_{\theta_a},f_{\theta_a'},f_{\theta_g},f_{\theta_g'}\) for the corresponding predictors.

\begin{lemma}
\label{lem:mix}
For any \(z=(x,y)\) and nonnegative \(\{\gamma_a\}\) with \(\sum_{a=1}^{|\mathcal{A}|}\gamma_a=1\),
\begin{equation}
\begin{aligned}
&\big|\ell(\textstyle\sum_{a=1}^{|\mathcal{A}|}\gamma_a f_{\theta_a}(x),y)-
      \ell(\sum_{a=1}^{|\mathcal{A}|}\gamma_a f_{\theta_a'}(x),y)\big| \\&
\le
\sum_{a=1}^{|\mathcal{A}|}\gamma_a\,
\big|\ell(f_{\theta_a}(x),y)-\ell(f_{\theta_a'}(x),y)\big|.
\end{aligned}
\end{equation}
\end{lemma}

\begin{proof}[Proof of Lemma~\ref{lem:mix}]
Fix $z=(x,y)$ and let $\phi(u):=\ell(u,y)$. Define
\begin{equation}
u_a := f_{\theta_a}(x),\; 
v_a := f_{\theta_a'}(x),\; 
U := \sum_{a=1}^{|\mathcal{A}|}\gamma_a\,u_a,\; 
V := \sum_{a=1}^{|\mathcal{A}|}\gamma_a\,v_a .
\end{equation}

We use the classical one-by-one swap technique to apply convexity once at each step to form the entire summation, starting from:

\begin{equation}
T_0:=V=\sum_{a=1}^{|\mathcal{A}|}\gamma_a v_a,\;
\end{equation}
and then scaling to:
\begin{equation}
T_k:=\sum_{j\le k}\gamma_j u_j+\sum_{j>k}\gamma_j v_j\quad (k=1,\dots,|\mathcal{A}|),
\end{equation}
so that $T_{{|\mathcal{A}|}}=U$ and

\begin{equation}
\begin{aligned}
&\phi(T_{{|\mathcal{A}|}})-\phi(T_0)
 = \sum_{a=1}^{|\mathcal{A}|}\bigl(\phi(T_k)-\phi(T_{k-1})\bigr)\\
&\Rightarrow\ |\phi(U)-\phi(V)|
 \le \sum_{a=1}^{|\mathcal{A}|}\bigl|\phi(T_k)-\phi(T_{k-1})\bigr|.
\end{aligned}
\end{equation}

Fix $k\in\{1,\dots,{|\mathcal{A}|}\}$ and write the common remainder as
\[
R_k:=\frac{1}{1-\gamma_k}\sum_{j\ne k}\gamma_j w_j,\;
w_j=\begin{cases}
u_j,& j<k,\\
v_j,& j>k.
\end{cases}
\]
Then
\[
T_k=(1-\gamma_k)R_k+\gamma_k u_k,\;
T_{k-1}=(1-\gamma_k)R_k+\gamma_k v_k.
\]
By convexity of $\phi$,
\[
\big|\phi(T_k)-\phi(T_{k-1})\big|\le \gamma_k\,\big|\phi(u_k)-\phi(v_k)\big|.
\]
Summing over $k$ and using the triangle inequality gives
\[
\big|\phi(U)-\phi(V)\big|
\le \sum_{a=1}^{|\mathcal{A}|}\gamma_k\,\big|\phi(u_k)-\phi(v_k)\big|.
\]
Unfolding $\phi$, $U$, and $V$ completes the proof:
\begin{equation}
\begin{aligned}
&\bigl|\ell(\textstyle\sum_{a=1}^{|\mathcal A|}\gamma_a f_{\theta_a}(x),y)
      -\ell(\sum_{a=1}^{|\mathcal A|}\gamma_a f_{\theta_a'}(x),y)\bigr| \\&
\le
\sum_{a=1}^{|\mathcal A|}\gamma_a\,
\bigl|\ell(f_{\theta_a}(x),y)-\ell(f_{\theta_a'}(x),y)\bigr|.
\end{aligned}
\end{equation}
.
\end{proof}

\begin{proof}[Proof of Theorem~\ref{thm:2}]
Now let's start to prove Theorem~\ref{thm:2}. The entire proof can be divided into the following four steps: \\
\noindent\textbf{Step 1: Three-term decomposition.}
For \(z=(x,y)\in\mathcal{T}\), define
\(\widetilde F_{\mathcal{D}}(x)=\sum_{a=1}^{|\mathcal{A}|}\gamma_a f_{\theta_a}(x)\) and
\(\widetilde F_{\mathcal{D}'}(x)=\sum_{a=1}^{|\mathcal{A}|}\gamma_a f_{\theta_a'}(x)\).
By the triangle inequality,
\begin{equation}
\begin{aligned}
\big|\ell(f_{\theta_g},z)-\ell(f_{\theta_g'},z)\big|
&\le
\underbrace{\big|\ell(f_{\theta_g},z)-\ell(\widetilde F_{\mathcal{D}},z)\big|}_{\text{(I)}} \\&
+
\underbrace{\big|\ell(\widetilde F_{\mathcal{D}},z)-\ell(\widetilde F_{\mathcal{D}'},z)\big|}_{\text{(II)}} \\&
+
\underbrace{\big|\ell(\widetilde F_{\mathcal{D}'},z)-\ell(f_{\theta_g'},z)\big|}_{\text{(III)}}.
\label{eq:three-terms}
\end{aligned}   
\end{equation}

\noindent\textbf{Step 2: Middle term via per-task \(\epsilon_a\)-stability.}
By Lemma~\ref{lem:mix} and Definition~2 applied within task \(a\),
\[
\text{(II)}\ \le\
\sum_{a=1}^{|\mathcal{A}|}\gamma_a\,\epsilon_a
\ :=\ \varepsilon_{\oplus}.
\]

\noindent\textbf{Step 3: End terms via a second-order mixing gap.}
We only use a \emph{local} regularity near the current iterates: once training has reached a certain level, the parameter
trajectory stays in a small neighborhood where (i) the loss has bounded prediction-gradient $L$, for predictions attained by the models; and (ii) along the short line segments that connect
\(\bar{\theta}\) to \(\theta_a\) and to \(\theta_g\), the network output admits a bounded parametric curvature with some constant \(M\).
Consequently,
\[
\text{(I)}\ \le\ L\,\big\|f_{\theta_g}(x)-\widetilde F_{\mathcal D}(x)\big\|,\qquad
\text{(III)}\ \le\ L\,\big\|\widetilde F_{\mathcal D'}(x)-f_{\theta_g'}(x)\big\|.
\]

\noindent\emph{Explicit Taylor expansions.}
For any $x$, expand $f_{\theta_a}(x)$ and $f_{\theta_g}(x)$ at $\bar{\theta}$ with the integral remainder:
\begin{equation}
\begin{aligned}
f_{\theta_a}(x)
&= f_{\bar{\theta}}(x)+J_{\bar{\theta}}(x)(\theta_a-\bar{\theta})
\\& + \underbrace{\int_0^1 (1-t)\,(\theta_a-\bar{\theta})^\top
   H_x\!\big(\bar{\theta}+t(\theta_a-\bar{\theta})\big)\,(\theta_a-\bar{\theta})\,dt}_{=:~r_a(x)},\\[2pt]
f_{\theta_g}(x)
&= f_{\bar{\theta}}(x)+J_{\bar{\theta}}(x)(\theta_g-\bar{\theta})
\\& + \underbrace{\int_0^1 (1-t)\,(\theta_g-\bar{\theta})^\top
   H_x\!\big(\bar{\theta}+t(\theta_g-\bar{\theta})\big)\,(\theta_g-\bar{\theta})\,dt}_{=:~r_g(x)},
\end{aligned}
\end{equation}
where $J_{\bar{\theta}}(x)$ is the Jacobian $\nabla_\theta f_\theta(x)|_{\theta=\bar{\theta}}$ and
$H_x(\cdot)$ is the parametric Hessian $\nabla^2_{\theta\theta} f_\theta(x)$.
By (ii), $\|r_a(x)\|\le \tfrac{M}{2}\|\theta_a-\bar{\theta}\|^2$ and
$\|r_g(x)\|\le \tfrac{M}{2}\|\theta_g-\bar{\theta}\|^2$.
Since $\theta_g=\sum_{a=1}^{|\mathcal A|}\gamma_a\theta_a$, the linear terms cancel, and thus
\begin{equation}
\begin{aligned}
&\big\|f_{\theta_g}(x)-\widetilde F_{\mathcal D}(x)\big\|
= \big\|r_g(x)-\textstyle\sum_{a=1}^{|\mathcal A|}\gamma_a r_a(x)\big\| \\&
\le \tfrac{M}{2}\!\left(\|\theta_g-\bar{\theta}\|^2+\sum_{a=1}^{|\mathcal A|}\gamma_a\|\theta_a-\bar{\theta}\|^2\right).
\end{aligned}
\end{equation}
Using $\|\theta_g-\bar{\theta}\|^2
= \big\|\sum_{a=1}^{|\mathcal A|}\gamma_a(\theta_a-\bar{\theta})\big\|^2
\le \sum_{a=1}^{|\mathcal A|}\gamma_a\|\theta_a-\bar{\theta}\|^2$ gives the compact bound
\[
\big\|f_{\theta_g}(x)-\widetilde F_{\mathcal D}(x)\big\|
\le M\sum_{a=1}^{|\mathcal A|}\gamma_a\|\theta_a-\bar{\theta}\|^2,
\]
and the same bound holds with $\mathcal D$ replaced by $\mathcal D'$.
Hence,
\begin{equation}
\begin{aligned}
\text{(I)}+\text{(III)}
\ \le\ 2LM\sum_{a=1}^{|\mathcal A|}\gamma_a\|\theta_a-\bar{\theta}\|^2.
\label{eq:mix-gap}
\end{aligned}
\end{equation}

\noindent\textbf{Step 4: Taking suprema to obtain uniform stability.}
Combining \eqref{eq:three-terms}, \eqref{eq:mix-gap} and taking the supremum over \(z\in\mathcal{T}\) and over all neighboring
\(\mathcal{D},\mathcal{D}'\) (differing in one example), we obtain the global uniform stability constant
\[
\varepsilon_g
\ \le\
\varepsilon_{\oplus}\;+\;C
\sum_{a=1}^{|\mathcal{A}|}\gamma_a\|\theta_a-\bar{\theta}\|^2, \quad C:=2\,L\,M,
\]
which matches the statement in Theorem~2. 
\end{proof}

\begin{figure*}[!t]
    \centering
  \begin{minipage}{0.48\linewidth}
    \makebox[0.48\linewidth]{\includegraphics[width=0.48\linewidth]{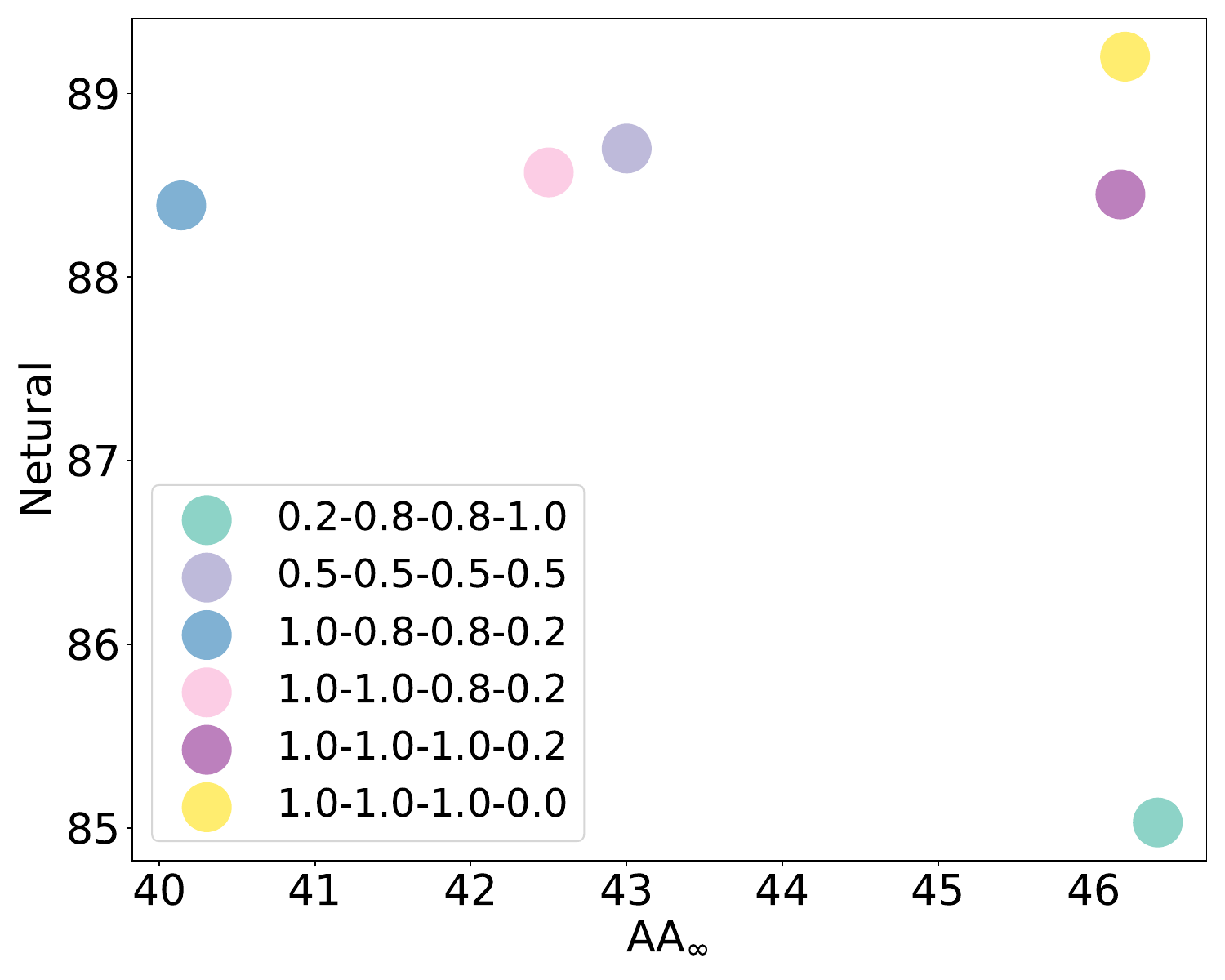}}
    \makebox[0.49\linewidth]{\includegraphics[width=0.49\linewidth]
{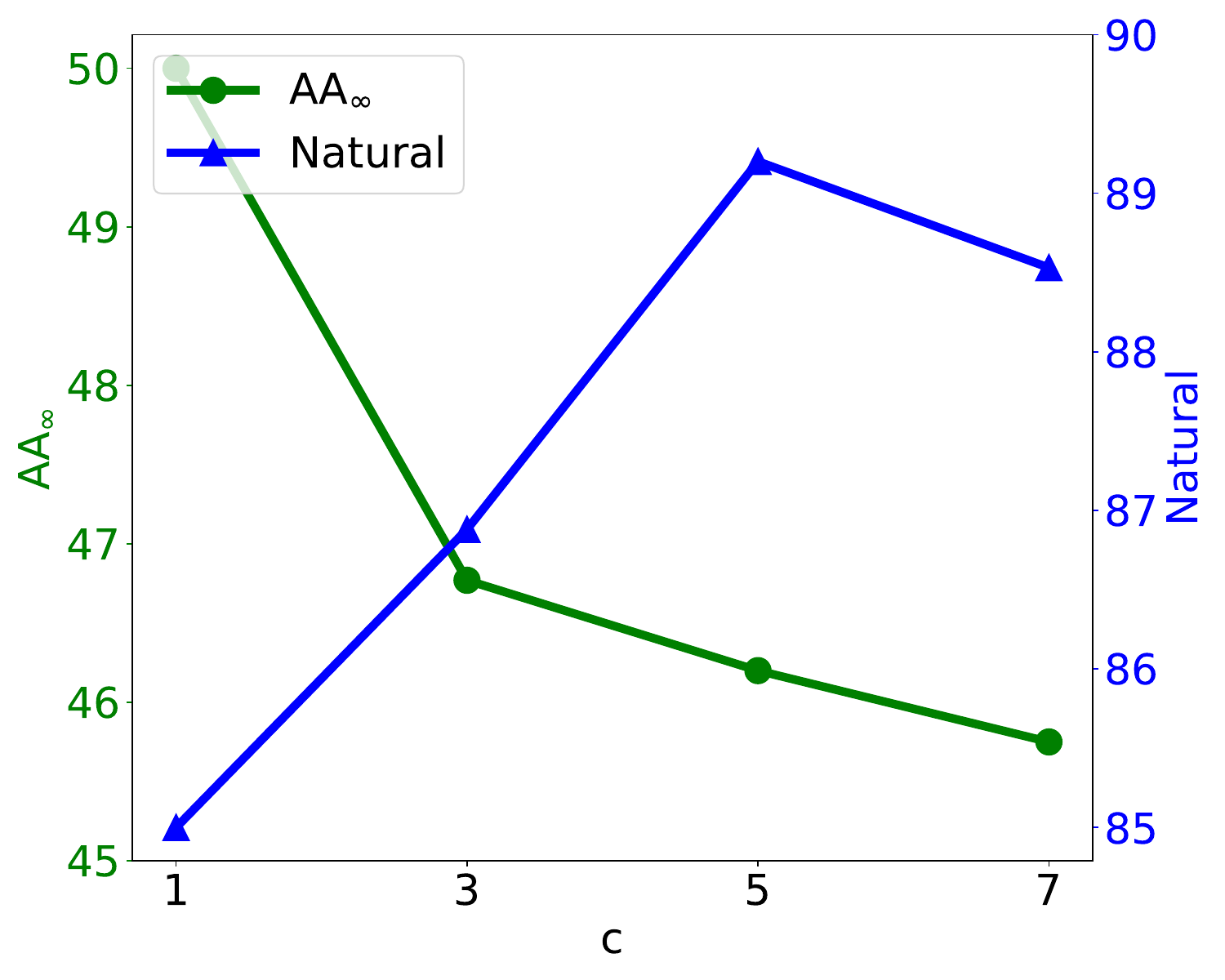}}%
    
    \makebox[1.0\linewidth]{\small (a) Generalist-D ($NT+\ell_\infty$)}%
  \end{minipage}%
    \hspace{15pt}
    \begin{minipage}{0.48\linewidth}
    \makebox[0.48\linewidth]{\includegraphics[width=0.48\linewidth]{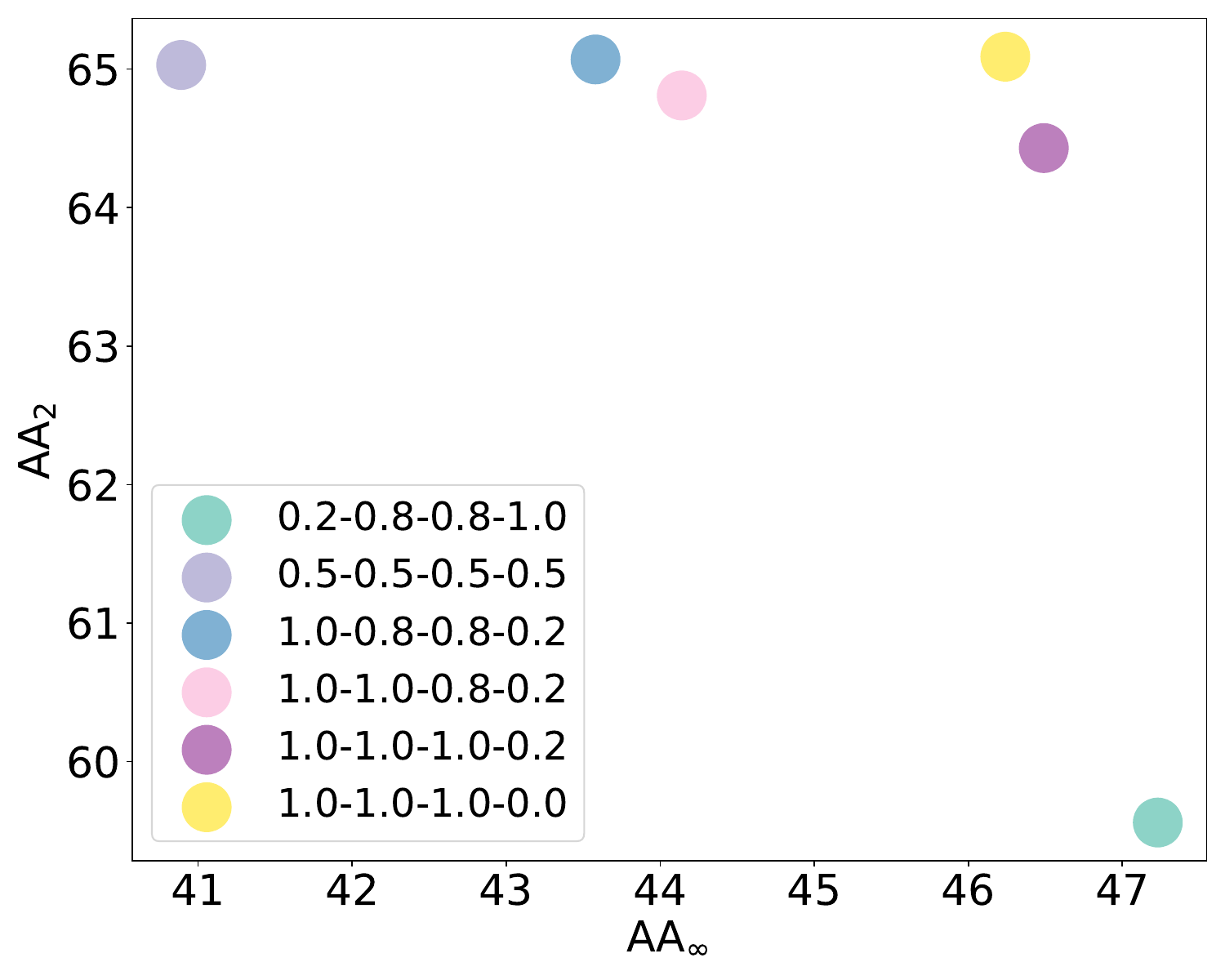}}%
    \hspace{2pt}
    \makebox[0.49\linewidth]{\includegraphics[width=0.49\linewidth]{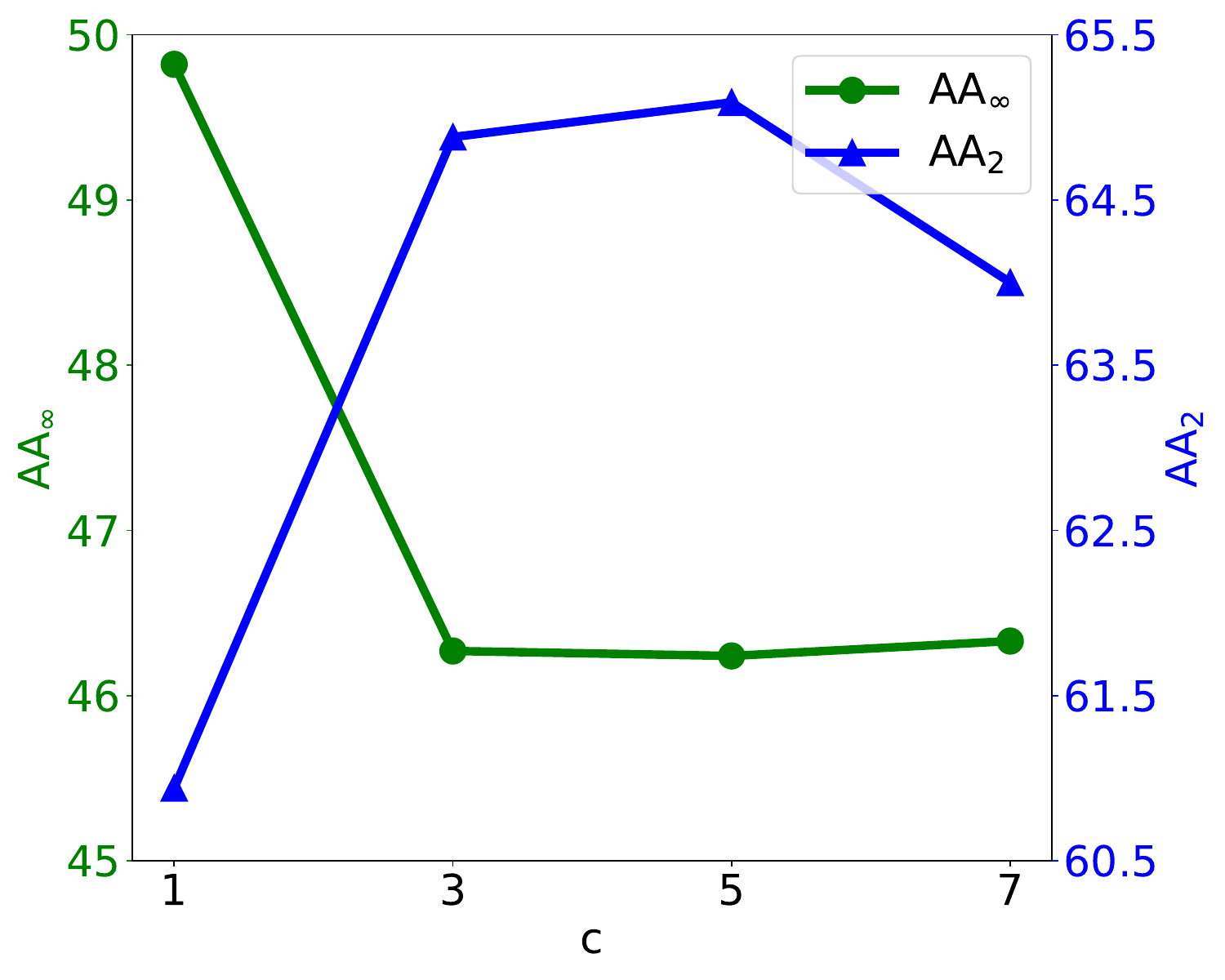}}
    \makebox[1.0\linewidth]{\small (b) Generalist-D ($\ell_\infty+\ell_2$)}%
  \end{minipage}%
    \caption{{The performances of Generalist-D with different mixing ratio strategies, \textit{i.e.} $\gamma_1$, and various values of communication frequency, \textit{i.e.} $c$, on the CIFAR-10 dataset. We evaluate Generalist-D ($NT+\ell_\infty$) with AA$_\infty$ and natural accuracy since it is designed to alleviate the natural-robustness tradeoff. For Generalist-D ($\ell_\infty+\ell_2$), AA$_\infty$ and AA$_2$ are chosen as metrics to investigate the influence of hyperparameter configurations on the robustness against multi-norm constraints.}}\label{fig:c_and_m}
    \vspace{-10pt}
\end{figure*}

\begin{figure}[t]
  \begin{minipage}{1.0\linewidth}
    \makebox[.5\linewidth]{\includegraphics[width=.5\linewidth]{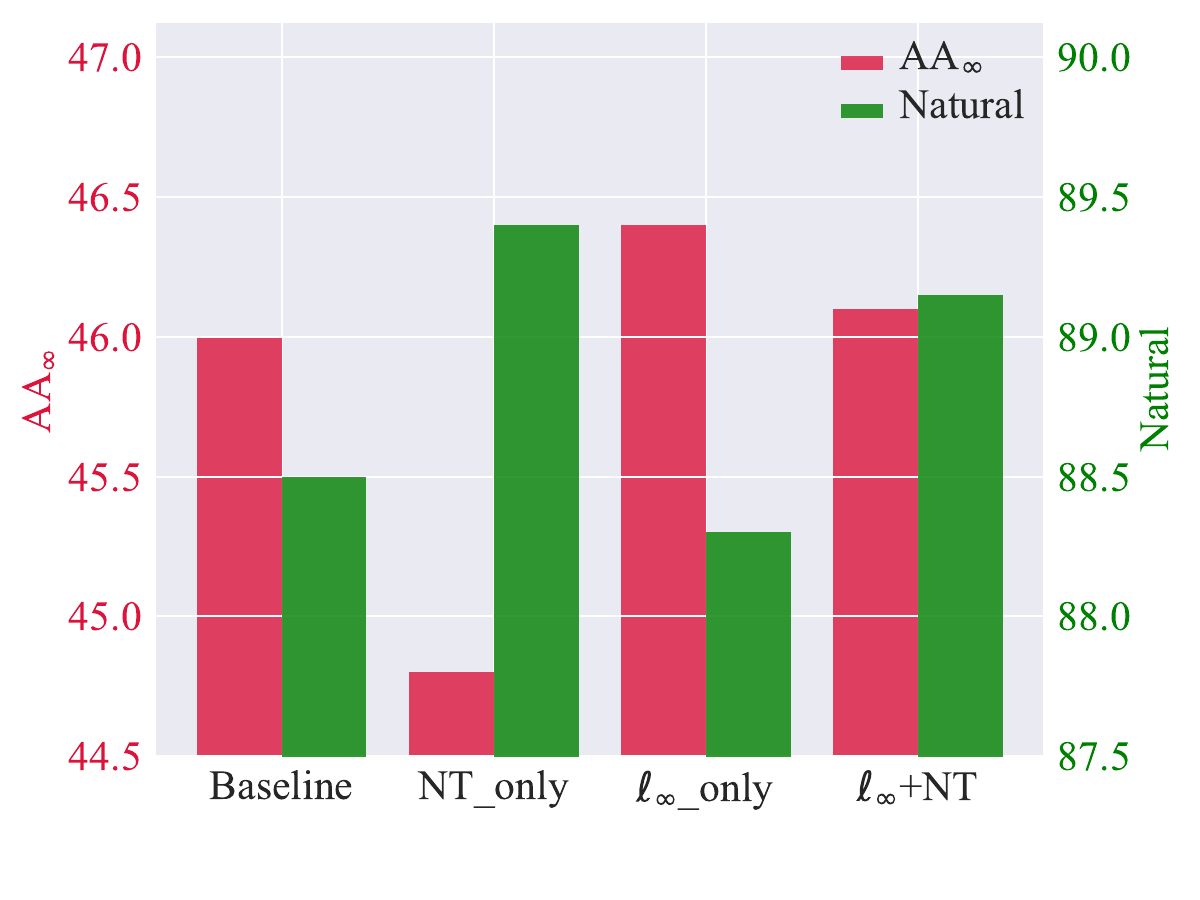}}%
    \makebox[.5\linewidth]{\includegraphics[width=.5\linewidth]{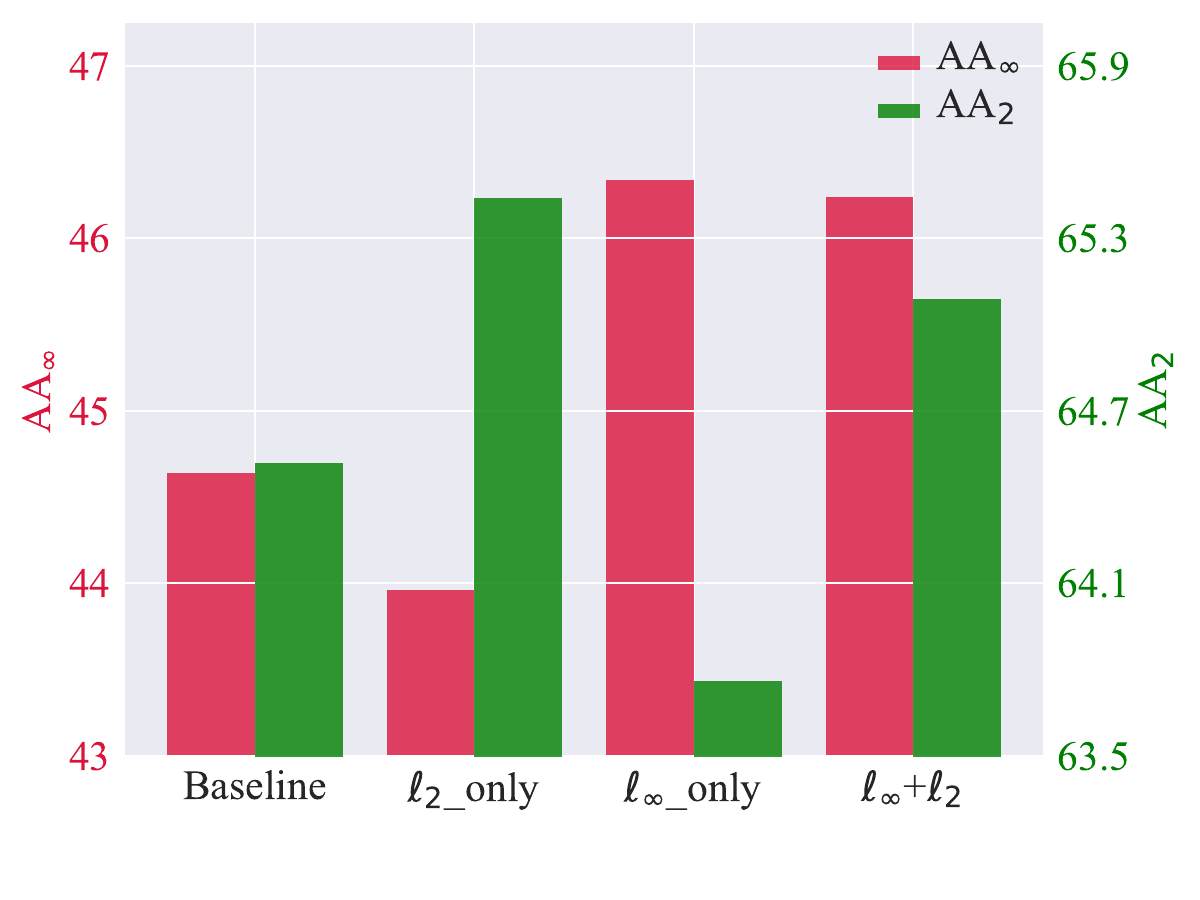}}%
    
    {\tiny\makebox[.5\linewidth]{\small{(a) Generalist-D ($NT+\ell_\infty$)}}}%
    {\tiny\makebox[.5\linewidth]{\small (b) Generalist-D ($\ell_\infty+\ell_2$)}}%
  \end{minipage}%

\caption{Base learners of Generalist-D applied with weight averaging on one or both of them. Using weight averaging through training can bring a performance boost in its corresponding sub-task, and thus has an effect on predictions of the global learner.}
\vspace{-10pt}
\label{fig:wa_chose}
\end{figure}

\section{Ablation Study for Generalist-D}
\label{app:abla_d}
Similar to Generalist-T, the mixing ratios and the communication frequency also control the trade-off of Generalist-D between the natural accuracy and robustness across norms. However, the difference is that the mixing ratio of Generalist-D is composed of only one scalar, $\gamma_1$, which is much easier for analysis. In the left images of both Figure \ref{fig:c_and_m} (a) and (b), we tune $\gamma_1$ with the same settings in Generalist-T. We have the exact same findings with those on the Generalist-T. Firstly, tuning  $\gamma_1$ in a descending order is the best choice if we aim at achieving satisfying performances in both perspectives. In addition, decaying the $\gamma_1$ early will also bring negative effects to the overall performances since noisy information will be brought to the overall framework if base learners are less specialized in their domains. The same phenomenon could also be extended from Generalist-T to Generalist-D regarding the communication frequency, $c$.

\vspace{-5pt}
\begin{figure}[t]
  \begin{minipage}{1.0\linewidth}
    \makebox[.5\linewidth]{\includegraphics[width=.5\linewidth]{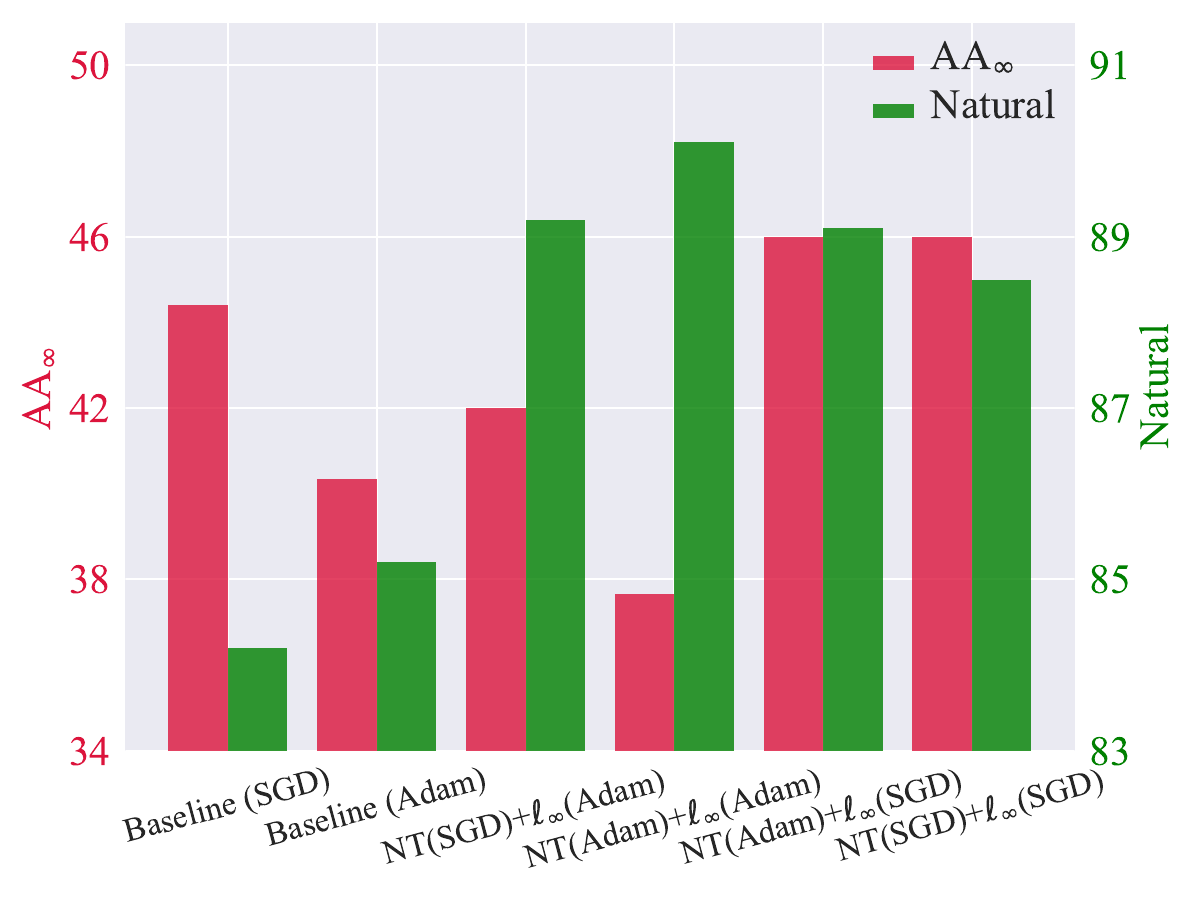}}%
    \hfill
    \makebox[.5\linewidth]{\includegraphics[width=.5\linewidth]{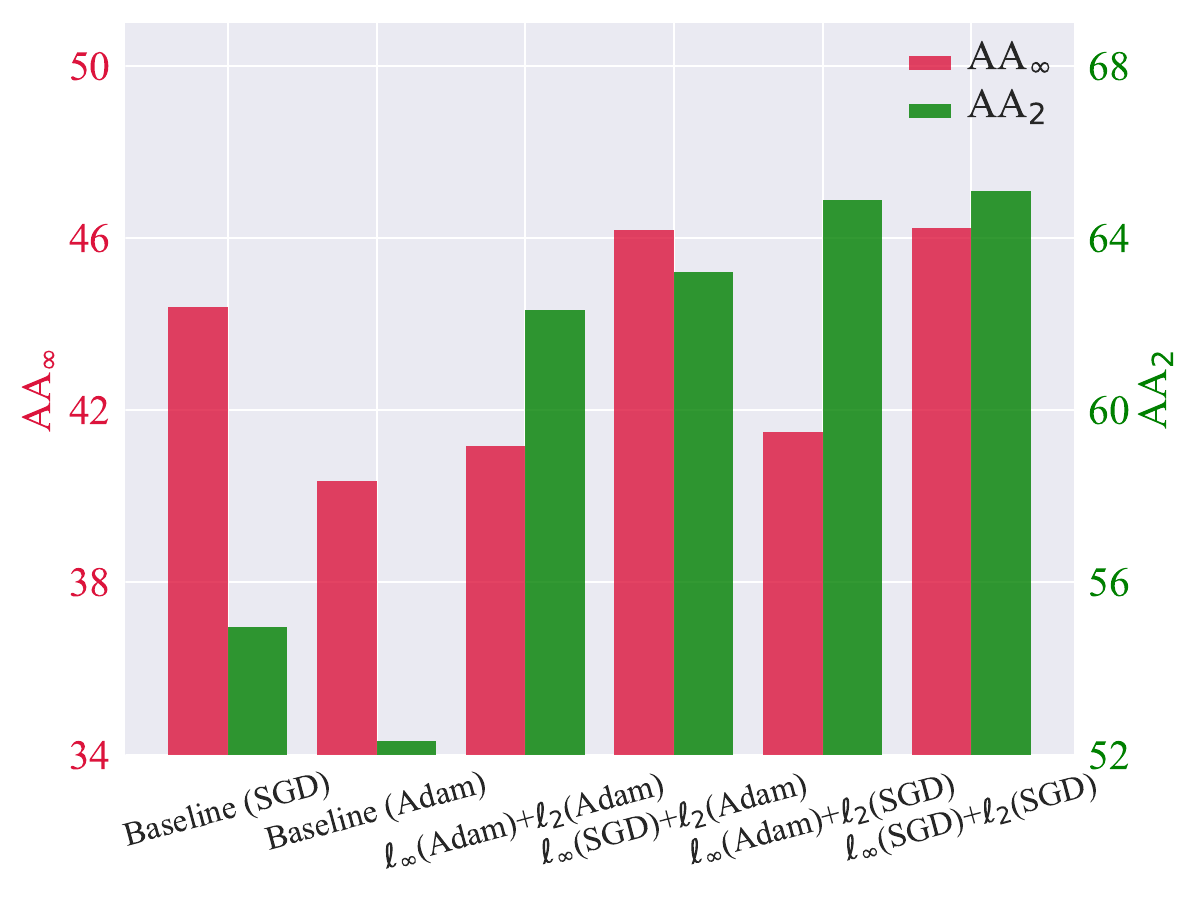}}%
    
    {\tiny\makebox[.5\linewidth]{\small{(a) Generalist-D ($NT+\ell_\infty$)}}}%
    {\tiny\makebox[.5\linewidth]{\small (b) Generalist-D ($\ell_\infty+\ell_2$)}}
  \end{minipage}%

\caption{ Base learners of Generalist-D optimized by different optimizers. The optimal selection is using Adam for the natural classification task but maintaining SGD for the adversarial training under the $\ell_\infty$ or $\ell_2$ norm.}\label{fig:optim_chose}
  \vspace{-10pt}
\end{figure}

\section{Customized Policies for Individual in Generalist-D} 
\label{app:double_policy}

In this section, we investigate customized policy for each base learners whether also work well for Generalist-D. Similar to Generalist-T, we study it from the perspective of weight averaging and different optimizer configurations.

\textbf{Weight Averaging.} As shown in Figure \ref{fig:wa_chose}, we evaluate the performance of the global learner with applying weight averaging on one base learner or all of them. The results manifest that when weight averaging is applied simultaneously to all base learners, we see an improvement in all aspects. Nevertheless, due to the influence of mismatched learning speeds, applying the weight averaging on a single learner will achieve unsatisfying performances in other aspects.

\textbf{Different Optimizers.} In Figure \ref{fig:optim_chose}, we also compare the performances of Generalist-D across diverse settings of optimizers. Comparing to AT with the SGD optimizer, AT with the Adam optimizer will compromise the robustness. In contrast, Adam is a better choice for the natural training. However, due to the decoupling property of Generalist-D, we can choose the customized optimizer for each base learner: it addresses the trade-off issue well by achieving outstanding performances in all dimensions. 

\end{document}